\documentclass{article}

\usepackage{iclr2027_conference,times}
\usepackage{hyperref}
\usepackage{url}
\usepackage{etoolbox}

\def\papermode{arxiv}

\newif\ifarxiv
\arxivfalse

\ifdefstring{\papermode}{arxiv}{
  \arxivtrue
  \iclrfinalcopy
}{}

\ifdefstring{\papermode}{camera}{
  \iclrfinalcopy
}{}

\makeatletter

\patchcmd{\@maketitle}
  {\lhead{Published as a conference paper at ICLR 2027}}
  {%
    \ifarxiv
      \lhead{Preprint.}
    \else
      \lhead{Published as a conference paper at ICLR 2027}
    \fi
  }
  {}
  {\PackageError{ICLR-arXiv}
    {Could not patch ICLR header}
    {The ICLR style file may have changed.}}

\makeatother

\usepackage{amsmath,amsfonts,bm}

\def\eqref#1{equation~\ref{#1}}

\def\1{\bm{1}}

\DeclareMathAlphabet{\mathsfit}{\encodingdefault}{\sfdefault}{m}{sl}
\SetMathAlphabet{\mathsfit}{bold}{\encodingdefault}{\sfdefault}{bx}{n}

\DeclareMathOperator*{\argmax}{arg\,max}

\usepackage{hyperref}
\usepackage{url}
\usepackage[utf8]{inputenc} 
\usepackage[T1]{fontenc}    
\usepackage{hyperref}       
\usepackage{url}            
\usepackage{booktabs}       
\usepackage{amsfonts}       
\usepackage{nicefrac}       
\usepackage{microtype}      
\usepackage{xcolor}         

\RequirePackage{fancyhdr}
\RequirePackage{algorithm}
\RequirePackage{algorithmic}
\RequirePackage{eso-pic} 
\RequirePackage{forloop}
\usepackage{amssymb}
\usepackage{mathtools}
\usepackage{amsthm}
\usepackage{cleveref}
\crefname{equation}{Eq.}{Eqs.}
\crefname{table}{Table}{Tables}
\crefname{figure}{Figure}{Figures}
\crefname{section}{Section}{Sections}
\crefname{algorithm}{Algorithm}{Algorithms}
\theoremstyle{plain}
\newtheorem{theorem}{Theorem}[section]
\newtheorem{proposition}[theorem]{Proposition}

\newtheorem{corollary}[theorem]{Corollary}
\theoremstyle{definition}

\theoremstyle{remark}

\usepackage{xcolor,colortbl}
\usepackage{adjustbox}
\usepackage{url}
\usepackage{multirow,multicol,xspace}
\usepackage{float}
\usepackage{graphics}
\usepackage{paralist}
\usepackage{wrapfig}
\usepackage[normalem]{ulem}
\usepackage{multirow}
\usepackage{adjustbox}
\usepackage[font=small,skip=6pt]{caption}

\newcommand{\MODEL}{\textsc{GraphFDM}\xspace}
\newcommand{\MODELNOSDC}{\ensuremath{\textsc{GraphFDM}_{\text{w/o SDC}}}}
\newcommand{\BEST}[1]{\textbf{\textcolor{purple}{#1}}}

\title{Graph Forward Distribution Matching for Molecular Inverse Design}

\author{%
    Yihan Zhu\\
    University of Notre Dame\\
    \texttt{yzhu25@nd.edu}\\
    \And 
    Yuhan Liu\\
    University of Notre Dame\\
    \texttt{yliu57@nd.edu}\\
    \And 
    Brett Savoie\\
    University of Notre Dame\\
    \texttt{bsavoie2@nd.edu}\\
    \And 
    Tengfei Luo\\
    University of Notre Dame\\
    \texttt{tluo@nd.edu}\\
    \And 
    Meng Jiang\\
    University of Notre Dame\\
    \texttt{mjiang2@nd.edu}\\
}

\begin{document}

\maketitle

\begin{abstract}
Achieving precise control over multiple properties without sacrificing chemical validity remains a central challenge in molecular inverse design. Existing reinforcement learning (RL) methods fine-tune graph diffusion models by treating \textbf{reverse} sampling as a sequential policy, using a single terminal reward to optimize hundreds of coupled decisions. They often suffer from instability, validity collapse, and limited property gains. 
We introduce \MODEL (Graph Forward Distribution Matching), a new online RL paradigm for graph diffusion that performs optimization through the \textbf{forward} process.
\MODEL uses valid generations to define a reward-tilted target distribution jointly optimized over graph size and molecular structure for each property condition, incorporating reinforcement signals into supervised learning without storing reverse trajectories. 
We derive the unique optimal target, prove a condition-wise improvement guarantee, and show that the fixed graph-size prior of standard graph diffusion leaves an irreducible matching gap.
In multi-conditional polymer and small-molecule generation, \MODEL achieves the lowest MAE on every target property, with reductions of up to 53.0\% relative to the strongest baselines and chemical validity above 0.99. It further generalizes to out-of-distribution property combinations. 
\end{abstract}

\section{Introduction}
\label{sec:introduction}
Graph diffusion models have emerged as a powerful approach for molecular inverse design, learning from large-scale unlabeled molecular data to generate chemically valid structures by denoising discrete atom and bond types~\citep{digress}.
However, they offer limited controllability over target properties under scarce labeled supervision~\citep{li2026disentangled,demodiff}. 
Practical inverse design in drug and materials discovery requires precise, simultaneous control over task-specific objectives such as synthetic accessibility and gas permeability, while preserving chemical validity.

Reinforcement learning (RL) provides a natural post-training paradigm to improve controllability and exploration~\citep{ddpo}. 
Existing RL approaches fine-tune graph diffusion models by treating reverse sampling as a sequential policy, using a single terminal reward to optimize hundreds of coupled decisions along each sampling trajectory, as shown in \cref{fig:mainfig} (a)~\citep{comole}. 
Under multi-property objectives, modifying a functional group, ring, or linker can shift properties simultaneously, and a single inconsistent bond can invalidate the entire structure. 
These coupled effects make reverse-policy updates prone to disrupting the structural priors learned during pretraining.
For example, GDPO~\citep{gdpo} and GraphGRPO~\citep{graphgrpo} drive validity close to zero while providing very limited property improvement in \cref{tab:main_results_gas}. 

We propose a new online RL framework for graph diffusion named Graph Forward Distribution Matching (\textbf{\MODEL}). 
\MODEL shifts policy optimization from reverse sampling trajectories to the forward process, incorporating terminal rewards into the denoising objective used in pretraining.
As illustrated in \cref{fig:mainfig} (b), 
\MODEL constructs a reward-tilted target by reweighting valid generations and matches it through denoising under the same sample weights.

Standard graph diffusion methods denoise node and edge types at a fixed graph size (the number of active nodes), sampled from a condition-independent prior~\citep{digress,graphdit}. 
However, molecular size is itself an important conditioning variable that influences property control and chemical validity~\citep{freegress,griddd}. Holding the size marginal fixed leaves an irreducible gap to the target distribution, which we formalize in Proposition~\ref{prop:structural_marginal_gap}.
We introduce \emph{Structural Distribution Control} (SDC) to match the complete terminal graph distribution, jointly optimizing the graph-size distribution and the conditional distribution of molecular structures at each size under separate KL constraints.

We derive the unique optimal target distribution and prove its condition-wise improvement guarantee (\cref{sec:what_to_match}).
We formalize the limitation of a fixed graph-size prior as an irreducible gap and show how SDC overcomes it (\cref{sec:structural_marginal_gap}).
In experiments, across a six-property molecular and a four-property polymer benchmark (\cref{tab:main_results_mol,tab:main_results_gas}), \MODEL outperforms the strongest baseline on every property, reducing MAE by up to 53.0\% and maintaining chemical validity above 0.99.
We show that \MODEL generalizes to unseen, out-of-distribution property combinations (\cref{fig:ood}).
Together, these results establish \MODEL as a principled and effective post-training framework for controllable molecular inverse design.
\begin{figure}[tbp]
  \centering
  \includegraphics[width=\linewidth]{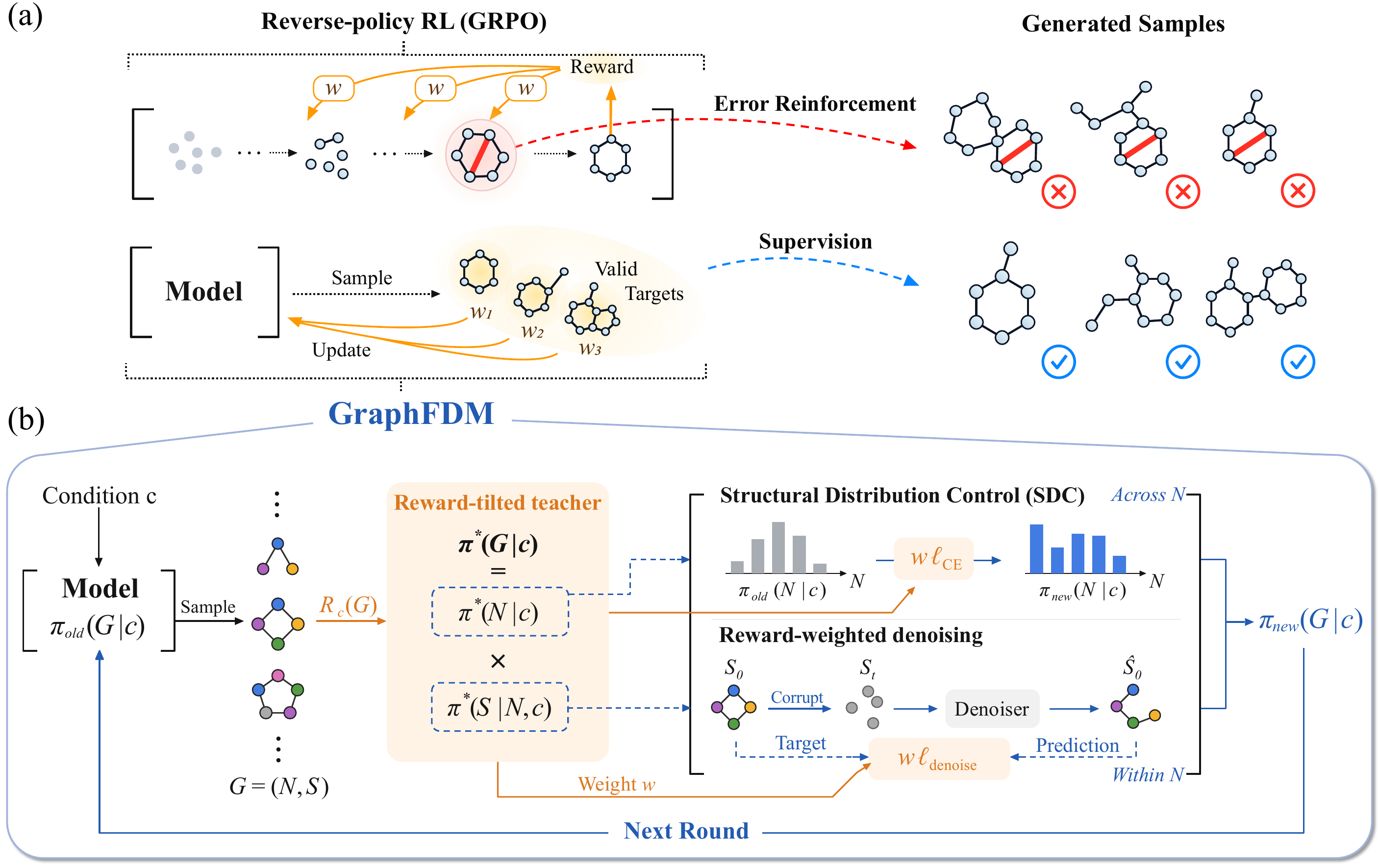}
    \caption{\textbf{Overview of \MODEL.}
    (a) Reverse-policy RL (e.g., GRPO) assigns terminal rewards across coupled reverse transitions which can reinforce structural errors.
    (b) At each online round, \MODEL constructs a reward-tilted teacher $\pi^{*}(G\mid c)$ over valid graphs $G=(N,S)$. 
    SDC reallocates probability across the number of active nodes $N$ while reward-weighted denoising refines molecular structures $S$ at each $N$.}
  \label{fig:mainfig}
\end{figure}

\section{Preliminaries}
\label{sec:preliminary}
\paragraph{Motif-Aware Graph Diffusion}
Let $\mathbf{c}\in\mathcal{C}\subseteq\mathbb{R}^{d}$ specify the desired values of $d$ molecular properties. 
Let $G\in\mathcal{G}$ denote a candidate molecular graph and $\mathcal{G}_{\mathrm{val}}\subseteq\mathcal{G}$ the set of valid molecular graphs. 
Molecular inverse design seeks a conditional distribution that assigns high probability to valid graphs whose
properties match $\mathbf{c}$.
We use the motif-aware discrete graph diffusion backbone of CoMole~\citep{comole}. Under its losslessly decodable representation, a clean molecular graph is written as
$ G=(N,S)$, with $S=(X,E,P)$,
where $N$ denotes the number of active graph units, each represented by
either an atom or motif in the vocabulary. $X$, $E$, and
$P$ denote the node, edge, and attachment-position labels, respectively. We set $S_0=S$ and represent $N$ by an active-node mask $m$ satisfying $N=\lVert m\rVert_1$. 
The mask remains fixed throughout diffusion. 
Let
$q_{t}(S_t\mid S_0,m)$ denote the forward corruption distribution
at timestep $t$. 
Given $(S_t,m,t,\mathbf c)$, the conditional denoiser
predicts the clean structure $S_0$ with the standard objective
\begin{equation}
\label{eq:denoise_obj}
\mathcal L_{\mathrm{den}}(\theta)
=
\mathbb E_{\substack{
(N,S_0,\mathbf c)\sim p_{\mathrm{data}},\;
t\sim\operatorname{Unif}(\{1,\ldots,T\}),\\
S_t\sim q_{t}(\cdot\mid S_0,m)
}}
\left[
\ell_{\mathrm{den}}^\theta
(S_0,S_t,m,t,\mathbf c)
\right], 
\end{equation}
where $\ell_{\mathrm{den}}^\theta$ is the sum of masked
categorical cross-entropies for $X$, $E$, and $P$.
During generation, $N\sim p_0(N)$ is sampled from a fixed, condition-independent empirical size prior.
Starting
from $S_T\sim p_T(\cdot\mid m)$, the reverse kernels
$p_\theta(S_{t-1}\mid S_t,m,t,\mathbf c)$ then induce the clean
conditional distribution $\pi_\theta(S\mid N,\mathbf c)$. Hence,
\begin{equation}
\label{eq:backbone_factorization}
\pi_\theta(G\mid\mathbf c)
=
p_0(N)\,
\pi_\theta(S\mid N,\mathbf c).
\end{equation}

\paragraph{Reward-Based Post-Training}
Reward-based post-training uses feedback from model-generated molecules to improve conditional control. For a target condition $\mathbf{c}$, let $R_{\mathbf c}:\mathcal G_{\mathrm{val}}\rightarrow\mathbb R$ denote a terminal reward that assigns larger values to valid molecules whose evaluated properties better match $\mathbf{c}$.
Reverse-policy optimization treats the denoising trajectory
$
S_T\rightarrow S_{T-1}\rightarrow\cdots\rightarrow S_0
$
as a finite-horizon policy and propagates the terminal reward through updates proportional to
$
R_{\mathbf{c}}(G)
\sum_{t=1}^{T}
\nabla_\theta
\log p_\theta
(S_{t-1}\mid S_t,m,t,\mathbf{c}).
$
The same terminal molecular feedback weights the transitions along the sampled reverse trajectory.

\section{Graph Forward Distribution Matching}
\label{sec:theory}
At online iteration $k$, \MODEL turns terminal molecular feedback
into a single structured teacher over the clean molecular graph \(G=(N,S)\). This teacher is then projected onto the native learning objectives for the structural marginal over \(N\) and the conditional molecular distribution over \(S\).
The current clean-graph generator factorizes as
\begin{equation}
\label{eq:current_graph_distribution}
\pi_k(G\mid\mathbf c)
=
p_k(N\mid\mathbf c)\,
\pi_{\theta_k}(S\mid N,\mathbf c),
\qquad
p_k(\cdot\mid\mathbf c)\ll p_0.
\end{equation}
We initialize \(p_k(N\mid\mathbf c)\) with the empirical prior \(p_0(N)\) and subsequently use the learned controller
\(p_k(N\mid\mathbf c)=p_{\phi_k}(N\mid\mathbf c)\).
Here, \(\ll\) restricts the controller to the structural support of the empirical prior.
Chemically valid graphs define the admissible support of the terminal teacher. 
We define the valid-conditioned
reference as
\begin{equation}
\label{eq:valid_reference}
\bar\pi_k(G\mid\mathbf c)
:=
\frac{
\pi_k(G\mid\mathbf c)\,
\mathbf 1\{G\in\mathcal G_{\mathrm{val}}\}
}{
\pi_k(\mathcal G_{\mathrm{val}}\mid\mathbf c)
},
\qquad
\pi_k(\mathcal G_{\mathrm{val}}\mid\mathbf c)>0.
\end{equation}
Writing
$
\bar\pi_k(N,S\mid\mathbf c)
=
\bar p_k(N\mid\mathbf c)\,
\bar\pi_k(S\mid N,\mathbf c),
$
$p_k$ is the structural marginal used to sample
$N$, whereas $\bar p_k$ is its valid-conditioned counterpart. 
They generally differ because
molecular validity can depend on $N$ (see~\cref{app:valid_structural_support}).

For any terminal teacher
\(\mu_k(\cdot\mid\mathbf c)\ll\bar\pi_k(\cdot\mid\mathbf c)\),
\MODEL defines the density ratio
\begin{equation}
\label{eq:generic_teacher_ratio}
w_k^\mu(G,\mathbf c)
:=
\frac{
\mu_k(G\mid\mathbf c)
}{
\bar\pi_k(G\mid\mathbf c)
}.
\end{equation}
\MODEL then re-corrupts clean graphs sampled from \(\bar\pi_k\) and
uses \(w_k^\mu\) to express the teacher’s denoising and structural
prediction risks as weighted expectations under the current reference.
Here, \emph{forward} refers to constructing supervision by reapplying
the fixed forward-corruption process. 
We first establish this teacher-agnostic projection,
then identify the structural-marginal gap of fixed-mask denoising, and
finally instantiate \(\mu_k=\pi_k^\star\) through SDC.

\subsection{Forward-denoising Matching}
\label{sec:how_to_match}
Factor the terminal teacher introduced above as
\[
\mu_k(N,S\mid\mathbf c)
=
\mu_{N,k}(N\mid\mathbf c)\,
\mu_{S,k}(S\mid N,\mathbf c).
\]
\MODEL transports both factors to their respective native learning
objectives using the single density ratio \(w_k^\mu\) defined in
\cref{eq:generic_teacher_ratio}.

\begin{proposition}[Forward-denoising matching]
\label{prop:forward_denoising_matching}
For any target condition \(\mathbf c\), assume that the losses below
are integrable under the corresponding teacher-induced distributions.
Then the weighted forward-denoising objective satisfies
\begin{equation}
\label{eq:forward_denoising_projection}
\begin{aligned}
\mathcal L_{S,k}^{\mu}(\theta;\mathbf c)
&:=
\mathbb E_{\substack{
G=(N,S_0)\sim\bar\pi_k(\cdot\mid\mathbf c),\\
t\sim\operatorname{Unif}(\{1,\ldots,T\}),\\
S_t\sim q_t(\cdot\mid S_0,m)
}}
\left[
w_k^\mu(G,\mathbf c)
\ell_{\mathrm{den}}^\theta
(S_0,S_t,m,t,\mathbf c)
\right]
\\
&=
\mathbb E_{\substack{
G=(N,S_0)\sim\mu_k(\cdot\mid\mathbf c),\\
t\sim\operatorname{Unif}(\{1,\ldots,T\}),\\
S_t\sim q_t(\cdot\mid S_0,m)
}}
\left[
\ell_{\mathrm{den}}^\theta
(S_0,S_t,m,t,\mathbf c)
\right],
\end{aligned}
\end{equation}
where \(m\) denotes the fixed structural mask associated with \(N\).

The same density ratio projects the structural marginal:
\begin{equation}
\label{eq:structural_projection_loss}
\begin{aligned}
\mathcal L_{N,k}^{\mu}(\phi;\mathbf c)
&:=
\mathbb E_{G\sim\bar\pi_k(\cdot\mid\mathbf c)}
\left[
w_k^\mu(G,\mathbf c)
\bigl(-\log p_\phi(N\mid\mathbf c)\bigr)
\right]
\\
&=
\mathbb E_{N\sim\mu_{N,k}(\cdot\mid\mathbf c)}
\left[
-\log p_\phi(N\mid\mathbf c)
\right].
\end{aligned}
\end{equation}
\end{proposition}

\Cref{eq:forward_denoising_projection} is an exact
condition-wise change of measure. It transports terminal teacher
preferences to every diffusion noise level by reapplying the
backbone's fixed forward-corruption process.
\Cref{eq:structural_projection_loss} transports the marginal of the
same teacher to the structural controller.

Let \(\mathbb E_{\mathbf c}\) denote the empirical average over target
conditions. Because the denoiser and structural controller share
parameters across conditions, their population objectives are
$$
\mathcal L_{S,k}^{\mu}(\theta)
:=
\mathbb E_{\mathbf c}
\left[
\mathcal L_{S,k}^{\mu}(\theta;\mathbf c)
\right],
\qquad
\mathcal L_{N,k}^{\mu}(\phi)
:=
\mathbb E_{\mathbf c}
\left[
\mathcal L_{N,k}^{\mu}(\phi;\mathbf c)
\right].
$$

Under simultaneous realizability, the population minimizers recover the teacher-induced denoising posterior marginals at every noise level and the structural marginal \(\mu_{N,k}\).
This is the precise population-level sense in which \MODEL projects and matches the teacher, see~\cref{app:proof_reward_to_denoising,app:proper_scoring}.

\subsection{Beyond Fixed-Mask Denoising: Structural-Marginal Gap}
\label{sec:structural_marginal_gap}

In traditional graph diffusion models, the structural variable
\(N\) is sampled before the reverse process and its mask
remains fixed throughout denoising. 
Any denoiser-only
update is restricted to the form
$
p_k(N\mid\mathbf c)\,
\widetilde\pi(S\mid N,\mathbf c),
$
leaving the structural marginal
\(p_k(N\mid\mathbf c)\) unchanged.
For the terminal teacher introduced above,
\(\mu_k\ll\bar\pi_k\ll\pi_k\), which implies
\(\mu_{N,k}(\cdot\mid\mathbf c)\ll
p_k(\cdot\mid\mathbf c)\).

\begin{proposition}[Structural-marginal gap]
\label{prop:structural_marginal_gap}
For any conditional molecular distribution
\(\widetilde\pi(S\mid N,\mathbf c)\) for which the divergence below is
finite, the KL chain rule gives
\begin{equation}
\label{eq:structural_kl_decomposition}
\begin{aligned}
&
D_{\mathrm{KL}}
\left(
\mu_k(\cdot\mid\mathbf c)
\Vert
p_k(N\mid\mathbf c)\,
\widetilde\pi(S\mid N,\mathbf c)
\right)
\\
&=
D_{\mathrm{KL}}
\left(
\mu_{N,k}(\cdot\mid\mathbf c)
\Vert
p_k(\cdot\mid\mathbf c)
\right)
\\
&+
\mathbb E_{N\sim\mu_{N,k}(\cdot\mid\mathbf c)}
D_{\mathrm{KL}}
\left(
\mu_{S,k}(\cdot\mid N,\mathbf c)
\Vert
\widetilde\pi(\cdot\mid N,\mathbf c)
\right).
\end{aligned}
\end{equation}
Consequently,
\begin{equation}
\label{eq:irreducible_structural_gap}
\begin{aligned}
&
\inf_{\widetilde\pi(S\mid N,\mathbf c)}
D_{\mathrm{KL}}
\left(
\mu_k(\cdot\mid\mathbf c)
\Vert
p_k(N\mid\mathbf c)\,
\widetilde\pi(S\mid N,\mathbf c)
\right)
=
D_{\mathrm{KL}}
\left(
\mu_{N,k}(\cdot\mid\mathbf c)
\Vert
p_k(\cdot\mid\mathbf c)
\right),
\end{aligned}
\end{equation}
where the infimum is attained by
$
\widetilde\pi(\cdot\mid N,\mathbf c)
=
\mu_{S,k}(\cdot\mid N,\mathbf c)
$
for \(\mu_{N,k}(\cdot\mid\mathbf c)\)-almost every \(N\).
\end{proposition}

Thus, even an oracle denoiser that exactly recovers the teacher
conditional \(\mu_{S,k}\) leaves a nonzero mismatch whenever
\(\mu_{N,k}(\cdot\mid\mathbf c)\neq
p_k(\cdot\mid\mathbf c)\).
This result motivates the structural
projection in \cref{eq:structural_projection_loss}: forward denoising learns teacher-aligned molecular structures
within each $N$. SDC, introduced next, unlocks reward-driven probability allocation
across $N$. The proof is provided in \cref{app:structural_gap}.

\subsection{Closing the Gap: Structural Distribution Control}
\label{sec:what_to_match}

Structural Distribution Control (SDC) addresses the gap established in
Prop.~\ref{prop:structural_marginal_gap}
by making the active-node-count
marginal part of the reward-based graph target. It separately controls
teacher movement across \(N\) and conditional molecular refinement
within each \(N\).

Let
$
\rho(N,S\mid\mathbf c)
=
\rho_N(N\mid\mathbf c)\,
\rho_S(S\mid N,\mathbf c)
$
be a candidate graph distribution satisfying
\(\rho(\cdot\mid\mathbf c)\ll
\bar\pi_k(\cdot\mid\mathbf c)\) for every target condition. Define
\begin{equation}
\label{eq:sdc_kl_components}
\begin{aligned}
\mathcal C_N(\rho)
&:=
\mathbb E_{\mathbf c}
D_{\mathrm{KL}}
\left(
\rho_N(\cdot\mid\mathbf c)
\Vert
\bar p_k(\cdot\mid\mathbf c)
\right),
\\
\mathcal C_S(\rho)
&:=
\mathbb E_{\mathbf c}
\mathbb E_{N\sim\rho_N(\cdot\mid\mathbf c)}
D_{\mathrm{KL}}
\left(
\rho_S(\cdot\mid N,\mathbf c)
\Vert
\bar\pi_k(\cdot\mid N,\mathbf c)
\right).
\end{aligned}
\end{equation}
By the KL chain rule,
\begin{equation}
\label{eq:structured_kl_chain_rule}
\mathbb E_{\mathbf c}
D_{\mathrm{KL}}
\left(
\rho(\cdot\mid\mathbf c)
\Vert
\bar\pi_k(\cdot\mid\mathbf c)
\right)
=
\mathcal C_N(\rho)+\mathcal C_S(\rho).
\end{equation}

For budgets \(\epsilon_N,\epsilon_S>0\), SDC instantiates the generic
\MODEL teacher as the solution of
\begin{equation}
\label{eq:sdc_trust_region_problem}
\begin{aligned}
\pi_k^\star
\in
\argmax_{\rho\ll\bar\pi_k}
\quad&
\mathbb E_{\mathbf c}
\mathbb E_{G\sim\rho(\cdot\mid\mathbf c)}
\left[
R_{\mathbf c}(G)
\right]
\quad 
\text{subject to}\quad&
\mathcal C_N(\rho)\leq\epsilon_N,
\quad
\mathcal C_S(\rho)\leq\epsilon_S.
\end{aligned}
\end{equation}
The two budgets separately constrain reward-driven reallocation across
\(N\) and conditional refinement within each \(N\).
For \(\tau_S>0\) and
\(n\in\operatorname{supp}\bar p_k(\cdot\mid\mathbf c)\), define the
conditional partition function
\begin{equation}
\label{eq:conditional_structural_partition}
Z_{S,k}(n,\mathbf c;\tau_S)
:=
\mathbb E_{S\sim\bar\pi_k(\cdot\mid n,\mathbf c)}
\left[
\exp\left(
\frac{R_{\mathbf c}(n,S)}{\tau_S}
\right)
\right].
\end{equation}

\begin{theorem}[Unique structured reward-tilted teacher]
\label{thm:structured_teacher}
Assume that the relevant graph and target-condition supports are finite
and that both constraints in
\cref{eq:sdc_trust_region_problem} are active at the optimum, with
positive optimal dual variables \(\tau_N,\tau_S\). Then the optimizer
is unique and factorizes, for every target condition, as
\[
\pi_k^\star(G\mid\mathbf c)
=
\pi_k^\star(N\mid\mathbf c)\,
\pi_k^\star(S\mid N,\mathbf c),
\]

where
\begin{equation}
\label{eq:structural_conditional_teacher}
\pi_k^\star(S\mid n,\mathbf c)
=
\frac{
\bar\pi_k(S\mid n,\mathbf c)
\exp\left(R_{\mathbf c}(n,S)/\tau_S\right)
}{
Z_{S,k}(n,\mathbf c;\tau_S)
},
\end{equation}
and
\begin{equation}
\label{eq:structural_marginal_teacher}
\pi_k^\star(N=n\mid\mathbf c)
=
\frac{
\bar p_k(n\mid\mathbf c)
Z_{S,k}(n,\mathbf c;\tau_S)^{\tau_S/\tau_N}
}{
Z_{N,k}(\mathbf c;\tau_N,\tau_S)
}.
\end{equation}
Here,
\begin{equation}
\label{eq:structural_marginal_partition}
Z_{N,k}(\mathbf c;\tau_N,\tau_S)
:=
\sum_{n\in\operatorname{supp}\bar p_k(\cdot\mid\mathbf c)}
\bar p_k(n\mid\mathbf c)
Z_{S,k}(n,\mathbf c;\tau_S)^{\tau_S/\tau_N}
\end{equation}
normalizes the structural marginal.
\end{theorem}

The conditional factor in
\cref{eq:structural_conditional_teacher} tilts molecular realizations
within each \(N\), whereas
\cref{eq:structural_marginal_teacher} reallocates the valid-conditioned
reference mass \(\bar p_k(n\mid\mathbf c)\) according to the aggregate
reward desirability
\(Z_{S,k}(n,\mathbf c;\tau_S)^{\tau_S/\tau_N}\).
The temperature \(\tau_S\) controls reward concentration within each
\(N\), while \(\tau_S/\tau_N\) directly controls the strength of
cross-\(N\) reallocation. The proof is provided in
\cref{app:structural_teacher}.
\paragraph{SDC projection weight.}
Substituting \(\mu_k=\pi_k^\star\) into
\cref{eq:generic_teacher_ratio} gives
\begin{equation}
\label{eq:structured_reward_weight}
\begin{aligned}
w_k(G,\mathbf c)
&:=
w_k^{\pi_k^\star}(G,\mathbf c)
=
\frac{
\pi_k^\star(G\mid\mathbf c)
}{
\bar\pi_k(G\mid\mathbf c)
}
=
\frac{
\exp\left(R_{\mathbf c}(G)/\tau_S\right)
Z_{S,k}(N,\mathbf c;\tau_S)^{\tau_S/\tau_N-1}
}{
Z_{N,k}(\mathbf c;\tau_N,\tau_S)
}.
\end{aligned}
\end{equation}
By Prop.~\ref{prop:forward_denoising_matching}, the same weight projects the
SDC teacher onto both native learning objectives. We henceforth write
\[
\mathcal L_{S,k}
:=
\mathcal L_{S,k}^{\pi_k^\star},
\qquad
\mathcal L_{N,k}
:=
\mathcal L_{N,k}^{\pi_k^\star}.
\]
\begin{corollary}[Condition-wise teacher improvement]
\label{cor:structured_teacher_improvement}
Under the assumptions of \cref{thm:structured_teacher}, define
$$
\Delta_k(\mathbf c)
:=
\mathbb E_{G\sim\pi_k^\star(\cdot\mid\mathbf c)}
\left[R_{\mathbf c}(G)\right]
-
\mathbb E_{G\sim\bar\pi_k(\cdot\mid\mathbf c)}
\left[R_{\mathbf c}(G)\right].
$$

Then, for every target condition \(\mathbf c\),
\begin{equation}
\label{eq:conditionwise_reward_improvement}
\begin{aligned}
\Delta_k(\mathbf c)
&\geq
\tau_N
D_{\mathrm{KL}}
\left(
\pi_k^\star(N\mid\mathbf c)
\Vert
\bar p_k(N\mid\mathbf c)
\right)
+
\tau_S
\mathbb E_{N\sim\pi_k^\star(N\mid\mathbf c)}
D_{\mathrm{KL}}
\left(
\pi_k^\star(S\mid N,\mathbf c)
\Vert
\bar\pi_k(S\mid N,\mathbf c)
\right)
\geq 0,
\end{aligned}
\end{equation}
with strict reward improvement whenever
\(\pi_k^\star(\cdot\mid\mathbf c)
\neq\bar\pi_k(\cdot\mid\mathbf c)\).
Moreover, because both constraints are active,
\begin{equation}
\label{eq:structured_reward_improvement}
\mathbb E_{\mathbf c}
\left[
\Delta_k(\mathbf c)
\right]
\geq
\tau_N\epsilon_N+\tau_S\epsilon_S
\geq 0.
\end{equation}
\end{corollary}
Since
\(\pi_k^\star\ll\bar\pi_k\),
the exact teacher also satisfies
$
\pi_k^\star(\mathcal G_{\mathrm{val}}\mid\mathbf c)=1
$
for every target condition. These reward and support guarantees apply
to the exact teacher. 
Its population projections are characterized by
Prop.~\ref{prop:forward_denoising_matching}. The proof and the corresponding
improvement identity are provided in
\cref{app:structured_teacher_improvement}.

\paragraph{Shared-temperature special case.}
If \(\tau_N=\tau_S=\tau\), multiplying
\cref{eq:structural_conditional_teacher,eq:structural_marginal_teacher}
recovers the standard joint reward tilt:
\begin{equation}
\label{eq:shared_temperature_target}
\pi_k^\star(G\mid\mathbf c)
=\frac{
\bar\pi_k(G\mid\mathbf c)
\exp\left(R_{\mathbf c}(G)/\tau\right)
}{
\mathbb E_{G'\sim\bar\pi_k(\cdot\mid\mathbf c)}
\left[
\exp\left(R_{\mathbf c}(G')/\tau\right)
\right]
}.
\end{equation}
Thus, ordinary reward tilting is the shared-temperature special case of
SDC.

\paragraph{Online update.}
At iteration \(k\), \MODEL samples candidate graphs from
\(\pi_k\), retains chemically valid candidates, evaluates their
terminal rewards, and estimates the SDC teacher weights \(w_k\).
It then updates \(\theta\) and \(\phi\) using empirical estimates of
\(\mathcal L_{S,k}\) and \(\mathcal L_{N,k}\), respectively. The
resulting generator factorizes as
\[
\pi_{k+1}(G\mid\mathbf c)
=
p_{\phi_{k+1}}(N\mid\mathbf c)\,
\pi_{\theta_{k+1}}(S\mid N,\mathbf c),
\]
and supplies the candidate distribution for the next iteration.
Finite-sample normalization, proposal correction, and dual estimation
are described in \cref{app:structured_projection}.

\section{Experiments}
\label{sec:experiments}
\textbf{RQ1}: We evaluate the controllable generation performance of \MODEL against molecular optimization, graph generation, and diffusion post-training baselines in \cref{sec:rq1}.
\textbf{RQ2}: We examine its out-of-distribution generalization and structural KL allocation in \cref{sec:rq2}.

\subsection{Experimental Setup}
\label{sec:exp-setup}
We evaluate \MODEL on two multi-property inverse-design benchmarks spanning small molecules and polymers.
We additionally report \textbf{SFT (Shared Init.)}, the common
supervised fine-tuned checkpoint used to initialize CoMole~\citep{comole} and both
\MODEL variants.
\textbf{\MODELNOSDC{}} uses the same forward-denoising updates as \MODEL, but without SDC, isolating the contribution of structural distribution control.
Model performance is validated across up to eleven metrics, including validity, distribution coverage, and conditional accuracy. 
Full dataset statistics and implementation details are provided in~\cref{app:dataset_details,app:experimental_details}.

\paragraph{Datasets}
For materials, we adopt the polymer dataset introduced in~\citet{graphdit}, with four target properties: three gas permeabilities (O$_2$Perm, CO$_2$Perm, N$_2$Perm) and synthetic accessibility score (SAscore) ~\citep{sascore}.
For molecules, we use QM9~\citep{qm9} and consider six target properties: 
dipole moment ($\mu$),
isotropic polarizability ($\alpha$), 
the energies of the highest occupied
and lowest unoccupied molecular orbitals
($\varepsilon_{\mathrm{HOMO}}$ and $\varepsilon_{\mathrm{LUMO}}$), 
heat capacity at constant volume ($C_v$), and SAscore.

\paragraph{Evaluation}
We use a fixed 6:2:2 split of each conditional benchmark into train/val/test sets.
Evaluations are conducted on 10,000 generated examples and report:
(1) post-processed validity (\textit{Valid.}), with raw validity measured before any rule checking or repair shown in parentheses;
(2) internal diversity (Div.); 
(3) fragment-based similarity with the reference set (Sim.);
(4)  Fr\'echet ChemNet Distance with the reference set (Dis.)~\citep{preuer2018frechet}; 
and MAE between conditioning targets and evaluated properties:
(5) synthetic accessibility (Synth.); 
(6)–(8)/(10) task-specific properties (Property). 
The oracle is a random forest trained on all task-related molecules~\citep{gao2022sample}, while SAscore is computed directly from molecular structure.
Oracle fitting accuracy and robustness to alternative evaluators are analyzed in \cref{app:oracle}.

\paragraph{Baselines}
We compare against a broad set of baselines spanning molecular optimization (GraphGA~\citep{graphga}, LSTM-HC~\citep{lstmhc}, MARS~\citep{mars}), diffusion-based generators (DiGress~\citep{digress}, GDSS~\citep{gdss}, GraphDiT~\citep{graphdit}, DeFoG~\citep{qin2024defog}, CSGD~\citep{csgd}, MELD~\citep{meld}), and methods with post-training (GDPO~\citep{gdpo}, GraphGRPO~\citep{graphgrpo}, VIDD~\citep{vidd}, CoMole~\citep{comole}) and an inference-time method (TreeDiff~\citep{treediff}).

\subsection{RQ1: Conditional Molecular Generation}
\label{sec:rq1}
\begin{table*}[t]
\centering
\small
\setlength{\tabcolsep}{5pt}
\caption{
\textbf{Multi-Conditional Generation of 10K Polymers:} 
Results on four properties (synthetic score, gas permeability for O$_2$, N$_2$, CO$_2$). 
MAE is computed between the input conditions and evaluated properties of the generated polymers.
Best results are in \BEST{red}.
}
\label{tab:main_results_gas}
\resizebox{\textwidth}{!}{%
\begin{tabular}{llcccccccc}
\toprule
\multirow{2}{*}{\textbf{Model}}
& \multicolumn{1}{c}{\textbf{Validity}}
& \multicolumn{3}{c}{\textbf{Distribution Learning}}
& \multicolumn{5}{c}{\textbf{Condition Control}} \\
\cmidrule (lr){2-2} \cmidrule (lr){3-5} \cmidrule (lr){6-10}
& Valid. (raw) $\uparrow$
& Div. $\uparrow$
& Sim. $\uparrow$
& Dis. $\downarrow$
& Synth. $\downarrow$
& O$_2$Perm $\downarrow$
& N$_2$Perm $\downarrow$
& CO$_2$Perm $\downarrow$
& Avg. MAE $\downarrow$ \\
\midrule
GraphGA             & \BEST{1.000} (N.A.)     & 0.883 & 0.927 & 9.188 & 1.331 & 1.984 & 2.290 & 1.949 & 1.888 \\
LSTM-HC             & 0.991 (N.A.)     & 0.894 & 0.932 & 12.122 & 1.413 & 1.381 & 1.641 & 1.356 & 1.448 \\
MARS                & \BEST{1.000} (N.A.)     & 0.838 & 0.928 & 7.562 & 1.166 & 1.576 & 1.833 & 1.607 & 1.546 \\
DiGress             & 0.839 (0.117) & 0.918 & 0.515 & 20.256 & 2.104 & 1.543 & 1.861 & 1.477 & 1.746 \\
GDSS                & \BEST{1.000} (0.003) & 0.857 & 0.004 & 37.627 & 1.164 & 1.260 & 1.477 & 1.285 & 1.297 \\
DeFoG               & 0.326 (0.003) & 0.934 & 0.168 & 24.859 & 1.548 & 1.178 & 1.366 & 1.186 & 1.320 \\
GraphDiT            & 0.826 (0.855) & 0.890 & 0.891 & 6.643 & 1.297 & 0.763 & 0.891 & 0.771 & 0.931 \\
MELD                & 0.135 (0.042) & 0.897 & 0.406 & 28.876 & 1.768 & 1.305 & 1.517 & 1.280 & 1.468 \\
CSGD                & 0.282 (0.212) & 0.862 & 0.651 & 20.397 & 1.171 & 1.154 & 1.342 & 1.148 & 1.204 \\
\midrule
TreeDiff            & 0.985 (0.005) & \BEST{0.935} & 0.072 & 30.719 & 2.411 & 1.232 & 1.476 & 1.221 & 1.585 \\
VIDD                & \BEST{1.000} (0.008) & 0.920 & 0.284 & 26.819 & 2.291 & 1.513 & 1.802 & 1.461 & 1.767 \\
GDPO                & 0.354 (0.077) & 0.921 & 0.489 & 24.422 & 1.424 & 1.357 & 1.606 & 1.325 & 1.428 \\
GraphGRPO           & 0.251 (0.020) & 0.871 & 0.347 & 26.221 & 2.972 & 1.818 & 2.229 & 1.730 & 2.187 \\
CoMole         & 0.918 (0.838) & 0.845 & \BEST{0.944} & 9.410 & 0.389 & 0.816 & 0.946 & 0.819 & 0.743 \\
\midrule
SFT (Shared Init.)    & 0.921 (0.898) & 0.863 & 0.930 & \BEST{6.328} & 0.370 & 0.901 & 1.013 & 0.899 & 0.796 \\
\MODELNOSDC{}  (ours) & 0.946 (0.909) & 0.845 & 0.902 & 9.427 & 0.231 & 0.668 & 0.766 & 0.691 & 0.589 \\
\MODEL  (ours)        & 0.991 (\BEST{0.986}) & 0.834 & 0.923 & 9.642 & \BEST{0.183} & \BEST{0.594} & \BEST{0.664} & \BEST{0.597} & \BEST{0.509} \\
\bottomrule
\end{tabular}%
}
\end{table*}
\begin{table*}[t]
\centering
\small
\setlength{\tabcolsep}{5pt}
\caption{
\textbf{Multi-Conditional Generation of 10K QM9 Molecules:} 
Results on six properties (synthetic score and five QM9 properties). 
MAE is computed between the input conditions and evaluated properties of the generated molecules.
Overall Rank aggregates property-wise MAE ranks across units and scales.
Best results are in \BEST{red}.
}
\label{tab:main_results_mol}
\resizebox{\textwidth}{!}{%
\begin{tabular}{llcccccccccc}
\toprule
\multirow{2}{*}{\textbf{Model}}
& \multicolumn{1}{c}{\textbf{Validity}}
& \multicolumn{3}{c}{\textbf{Distribution Learning}}
& \multicolumn{7}{c}{\textbf{Condition Control}} \\
\cmidrule (lr){2-2} \cmidrule (lr){3-5} \cmidrule (lr){6-12}
& Valid. (raw) $\uparrow$
& Div. $\uparrow$
& Sim. $\uparrow$
& Dis. $\downarrow$
& Synth. $\downarrow$
& $\mu$ $\downarrow$
& $\alpha$ $\downarrow$
& $\varepsilon_{\mathrm{HOMO}}$ $\downarrow$
& $\varepsilon_{\mathrm{LUMO}}$ $\downarrow$
& $C_v$ $\downarrow$
& Overall Rank $\downarrow$ \\
\midrule
GraphGA             & \BEST{1.000} (N.A.) & 0.927 & 0.922 & 0.777 & 0.946 & 0.702 & 4.903 & 0.019 & 0.020 & 2.121 & 8 \\
LSTM-HC             & 0.990 (N.A.)     & 0.916 & \BEST{0.957} & 1.400 & 1.016 & 1.478 & 8.489 & 0.022 & 0.052 & 4.300 & 16 \\
MARS                & \BEST{1.000} (N.A.) & 0.931 & 0.770 & 5.731 & 1.037 & 0.844 & 7.597 & 0.014 & 0.021 & 2.961 & 10 \\
DiGress             & 0.915 (0.849) & 0.918 & 0.941 & 4.696 & 1.612 & 0.912 & 3.897 & 0.012 & 0.018 & 1.900 & 7\\
GDSS                & \BEST{1.000} (0.737) & 0.908 & 0.733 & 6.559 & 1.340 & 1.804 & 10.371 & 0.021 & 0.052 & 3.586 &16 \\
DeFoG               & 0.903 (0.578) & \BEST{0.932} & 0.890 & 1.521 & 0.823 & 1.038 & 7.319 & 0.016 & 0.024 & 3.762 & 12\\
GraphDiT            & 0.913 (0.720) & 0.924 & 0.875 & 2.407 & 0.701 & 1.063 & 5.306 & 0.015 & 0.025 & 2.529 & 9 \\
MELD                & 0.823 (0.645) & 0.916 & 0.905 & 1.377 & 0.893 & 1.152 & 7.077 & 0.020 & 0.030 & 3.112 & 13 \\
CSGD                & 0.971 (0.875) & 0.907 & 0.936 & 2.147 & 0.823 & 1.119 & 5.664 & 0.017 & 0.031 & 2.647 & 11 \\
\midrule
TreeDiff & 0.974 (0.088) & 0.911 & 0.662 & 4.394 & 1.264 & 1.477 & 5.480 & 0.019 & 0.054 & 4.252 & 14 \\
VIDD                & \BEST{1.000} (0.550) & 0.921 & 0.790 & 6.157 & 1.804 & 1.395 & 8.223 & 0.020 & 0.057 & 3.700 & 15 \\
GDPO                & 0.910 (0.860) & 0.917 & 0.947 & 0.710 & 0.542 & 0.908 & 3.748 & 0.012 & 0.018 & 1.886 & 4 \\
GraphGRPO           & 0.976 (0.841) & 0.911 & 0.904 & 3.114 & 0.572 & 0.862 & 3.528 & 0.011 & 0.018 & 1.992 & 3 \\
CoMole              & 0.962 (0.866) & 0.916 & 0.927 & 4.231 & 0.477 & 0.882 & 4.322 & 0.012 & 0.021 & 1.929 & 5 \\
\midrule
SFT (Shared Init.)         & 0.895 (0.753) & 0.917 & 0.942 & \BEST{0.550} & 0.537 & 0.866 & 4.570 & 0.013 & 0.022 & 2.073 & 6 \\
\MODELNOSDC{}  (ours) & 0.973 (0.921) & 0.917 & 0.907 & 0.689 & 0.400 & 0.782 & 3.411 & 0.010 & 0.015 & 1.622 & 2 \\
\MODEL  (ours)        & 0.994 (\BEST{0.960}) & 0.917 & 0.902 & 0.615 & \BEST{0.359} & \BEST{0.697} & \BEST{3.307} & \BEST{0.009} & \BEST{0.015} & \BEST{1.473} & \BEST{1} \\
\bottomrule
\end{tabular}%
}
\end{table*}

\textbf{As shown in~\cref{tab:main_results_gas,tab:main_results_mol}}, \MODEL \textbf{achieves the best controllability on all target properties and maintains validity over 0.99}.
\paragraph{Chemical Validity.}
Rule-filtered validity can obscure the intrinsic feasibility of generated structures. 
For example, in \cref{tab:main_results_gas}, GDSS and VIDD both report $1.0$ validity, yet their raw validity collapses to $0.003$ and $0.008$, respectively. 
The post-training baselines GDPO and GraphGRPO remain weak under both evaluations, attaining only $0.354 (0.077)$ and $0.251 (0.020)$ validity.
In contrast, \MODELNOSDC{} achieves raw validity of $0.909$ on Gas and $0.921$ on QM9, which \MODEL further improves to $0.986$ and $0.960$. 
Forward distribution matching preserves intrinsic chemical feasibility under stronger property optimization, with SDC providing a further consistent gain.

\paragraph{Distribution Learning.}
A central concern in property optimization is whether improved conditional accuracy comes at the cost of drifting away from the learned data distribution.
On the polymer benchmark, SFT (Shared Init.) and GraphDiT achieve low distribution distances of $6.328$ and $6.643$, while several optimization baselines exhibit substantial distribution shift, with fragment similarity below $0.5$ and distribution distance above $20$. 
\MODEL maintains high fragment similarity ($0.923$) with a moderate distribution distance ($9.642$).
On QM9, \MODEL retains a low distribution distance of $0.615$, close to the shared SFT initialization ($0.550$), while maintaining high diversity ($0.917$). These results suggest that \MODEL improves controllability without severe distributional drift or mode collapse. 

\paragraph{Condition Controllability.}
\MODEL provides the strongest overall condition control on both benchmarks. 
In \cref{tab:main_results_gas}, starting from the shared SFT initialization ($0.796$), \MODELNOSDC{} reduces Avg.~MAE to $0.589$, while \MODEL further reduces it to $0.509$, outperforming the strongest baseline CoMole ($0.743$) by $31.5\%$. 
\MODEL achieves the lowest MAE for all four properties.
In \cref{tab:main_results_mol}, the same progression improves Overall Rank from $6$ for SFT to $2$ and $1$ for \MODELNOSDC{} and \MODEL, ahead of GraphGRPO and GDPO. 
\MODEL obtains the best MAE on all six properties, including SAS ($0.359$ vs.\ CoMole's $0.477$), $\alpha$ ($3.307$ vs.\ GraphGRPO's $3.528$), and $C_v$ ($1.473$ vs.\ GDPO's $1.886$). 
Together, these results isolate the gains from forward distribution matching and the additional contribution of SDC. SDC further improves all four polymer errors and five of six QM9 errors while matching the remaining one.

\subsection{RQ2: Ablation Studies and Model Analysis}
\label{sec:rq2}

\paragraph{Generalization.}
We evaluate compositional out-of-distribution generalization on 200 unseen QM9 property targets. 
We partition each of the five quantum properties into low, medium, and high bins using training-set quantiles. 
We retain joint bin combinations occurring in fewer than $0.1\%$ of training molecules and select the eight rarest eligible signatures.
We randomly sample 25 targets from each combination and include SAScore as the sixth condition (see \cref{app:dataset_details}).
Without further training, \MODEL achieves the lowest MAE on every property while maintaining high raw validity~(\cref{tab:ood}).
\Cref{fig:ood} jointly embeds target vectors and per-target centroids of generated properties. 
\MODEL more consistently follows the target modes, whereas competing methods exhibit larger deviations or incomplete coverage of the requested property space.
These results suggest that \MODEL can compose learned property controls to generalize to rare or unseen joint property combinations.
\begin{figure}[tbp]
  \centering
  \includegraphics[width=\linewidth]{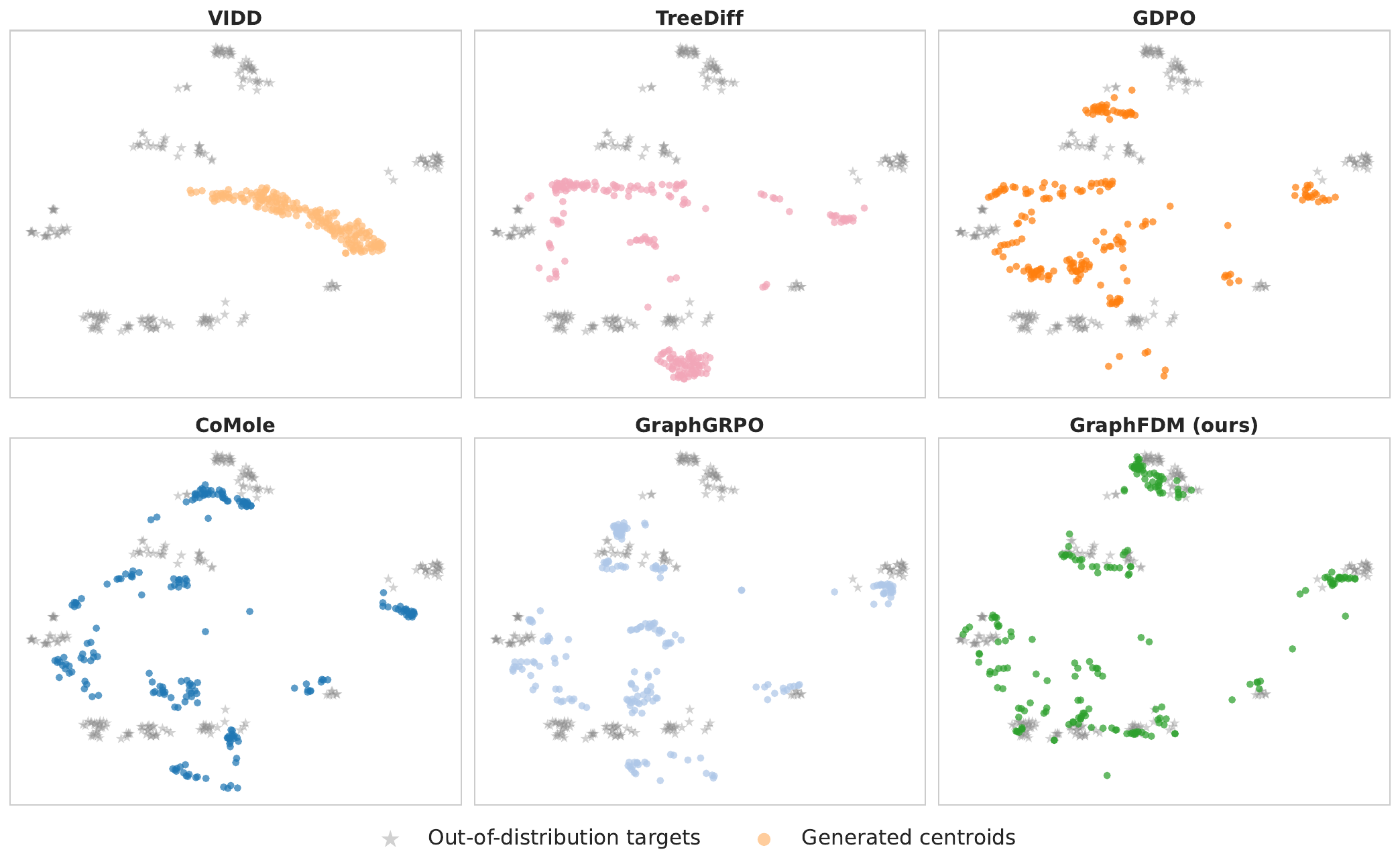}
    \caption{\textbf{Out-of-distribution multi-property alignment.}
    Joint t-SNE visualization of normalized target vectors (gray stars) and per-target centroids of valid generated molecules (colored circles).
    \MODEL more consistently follows the target modes than competing post-training methods.}
  \label{fig:ood}
\end{figure}
\begin{table*}[t]
\centering
\small
\setlength{\tabcolsep}{5pt}
\caption{\textbf{Out-of-distribution Multi-Conditional Generation of 10K QM9 Molecules:}
We evaluate post-training methods on 200 unseen targets with rare property combinations.
Best results are in \BEST{red}.
}
\label{tab:ood}
\resizebox{\textwidth}{!}{%
\begin{tabular}{llccccccccc}
\toprule
\multirow{2}{*}{\textbf{Model}}
& \multicolumn{1}{c}{\textbf{Validity}}
& \multicolumn{3}{c}{\textbf{Distribution Learning}}
& \multicolumn{6}{c}{\textbf{Condition Control}} \\
\cmidrule (lr){2-2} \cmidrule (lr){3-5} \cmidrule (lr){6-11}
& Valid. (raw) $\uparrow$
& Div. $\uparrow$
& Sim. $\uparrow$
& Dis. $\downarrow$
& Synth. $\downarrow$
& $\mu$ $\downarrow$
& $\alpha$ $\downarrow$
& $\varepsilon_{\mathrm{HOMO}}$ $\downarrow$
& $\varepsilon_{\mathrm{LUMO}}$ $\downarrow$
& $C_v$ $\downarrow$\\
\midrule

VIDD
& \BEST{1.000} (0.535) & 0.923 & 0.664 & 10.985
& 2.002 & 1.262 & 10.708 & 0.032 & 0.059 & 4.426  \\

TreeDiff
& 0.944 (0.038) & \BEST{0.934} & 0.687 & 5.488 & 0.996 & 1.229 & 9.879 & 0.021 & 0.031 & 3.372 \\

GDPO
& 0.804 (0.717) & 0.931 & \BEST{0.829} & 3.496
& 0.574 & 1.227 & 6.875 & 0.019 & 0.030 & 2.438  \\

GraphGRPO
& 0.955 (0.705) & 0.927 & 0.787 & 9.873
& 0.692 & 1.163 & 7.127 & 0.029 & 0.026 & 2.681  \\

CoMole
& 0.875 (0.716) & 0.926 & 0.744 & \BEST{3.207}
& 0.642 & 1.210 & 6.688 & 0.019 & 0.030 & 2.792  \\

\MODEL\ (ours)
& 0.973 (\BEST{0.892}) & 0.920 & 0.701 & 3.349
& \BEST{0.449} & \BEST{1.059} & \BEST{5.275}
& \BEST{0.015} & \BEST{0.023} & \BEST{1.870} \\
\bottomrule
\end{tabular}%
}
\end{table*}

\paragraph{KL Budget Allocation in SDC.}
\label{sec:sdc_geometry}
Following \cref{eq:sdc_kl_components,eq:structured_kl_chain_rule}, we study how to divide a fixed KL budget between graph-size reallocation and molecular refinement at each size.
Here, $\mathcal C_N:\mathcal C_S$ denotes the ratio of KL changes in the structural marginal over graph sizes and the conditional molecular distributions.
We compare a joint constraint that allows an adaptive split with separate constraints assigning budgets in ratios of $1{:}1$ and $1{:}2$.
All variants use the same total KL radius and training configuration.
As shown in~\cref{tab:sdc_geometry}, the joint constraint yields
$\mathcal C_N:\mathcal C_S\approx2.5{:}1$, 
allocating more distributional change to the structural marginal.
Balancing the allocation at
$1{:}1$ reduces Avg.~MAE from $0.590$ to $0.509$ ($13.7\%$), with
slightly improved validity and Avg.~Gas MAE dropping from $0.725$ to $0.618$. 
The $1{:}2$ allocation also reduces Avg.~MAE relative to the joint constraint but is less effective than $1{:}1$.
We therefore use the $1{:}1$ allocation in our main experiments.
\begin{table*}[t]
\centering
\caption{
\textbf{KL budget allocation in SDC:}
$\mathcal C_N:\mathcal C_S$ denotes the ratio of structural-marginal to
conditional-structure KL change.
Avg. Gas averages the three gas-permeability MAEs. Best results are in \textbf{bold}.}
\label{tab:sdc_geometry}
\resizebox{\textwidth}{!}{%
\small
\setlength{\tabcolsep}{5pt}
\begin{tabular}{l c c ccc cc}
\toprule
\multirow{2}{*}{\textbf{Constraint}}
& \multirow{2}{*}{\textbf{$\mathcal C_N:\mathcal C_S$}}
& \multicolumn{1}{c}{\textbf{Validity}}
& \multicolumn{2}{c}{\textbf{Distribution}}
& \multicolumn{3}{c}{\textbf{Condition Control}}\\
\cmidrule(lr){3-3}
\cmidrule(lr){4-5}
\cmidrule(lr){6-8}
&
& Valid. (raw) $\uparrow$
& Div. $\uparrow$
& Dis. $\downarrow$
& Synth. $\downarrow$
& Avg. Gas $\downarrow$
& Avg. MAE$\downarrow$ \\
\midrule
Joint KL
& Adaptive ($\approx2.5{:}1$)
& 0.988 (0.979)
& 0.834 & \textbf{8.764}
& 0.184 & 0.725 & 0.590 \\

Separate 1:1
& $1{:}1$
& \textbf{0.991 (0.986)}
& 0.834 & 9.642
& \textbf{0.183} & \textbf{0.618} & \textbf{0.509} \\

Separate 1:2
& $1{:}2$
& 0.986 (0.965)
& \textbf{0.840} & 9.418 
& 0.189 & 0.648 & 0.534
\\
\bottomrule
\end{tabular}
}
\end{table*}

\section{Related Work}
\label{sec:related_work}
\paragraph{Molecular Graph Generation.}
Molecular inverse design has been approached through evolutionary search, autoregressive generation, and latent-space optimization~\citep{graphga,mars,lstmhc,jtvaebo}. 
Discrete graph diffusion models generate categorical node and edge variables~\citep{digress}.
GraphDiT~\citep{graphdit} incorporates multiple property conditions into graph denoising, while GrIDDD~\citep{griddd} supports node insertion and deletion during diffusion.
DeFoG~\citep{qin2024defog} formulates graph generation through discrete flow matching. 
Reward-guided post-training adapts pretrained graph generators to task-specific objectives. GDPO~\citep{gdpo} and CoMole~\citep{comole} optimize reverse policies, with CoMole operating in a motif-aware graph space.
GraphGRPO~\citep{graphgrpo} applies group-relative policy optimization through analytic graph-flow transitions.

\paragraph{Diffusion Post-Training.}
Recent diffusion alignment methods incorporate terminal feedback at different levels of the generative process.
For discrete sequence models, DDPP~\citep{ddpp} approximates reward-conditioned posteriors across masking levels, while DMPO~\citep{dmpo} formulates diffusion-language reasoning as policy-distribution matching. 
For continuous visual generators, DiffusionNFT~\citep{zheng2026diffusionnft} contrasts positive and negative generations through forward flow matching, and AWM~\citep{awm} derives a policy-gradient-consistent surrogate from
score- or flow-matching losses.
TMPO~\citep{tmpo} matches sampled reverse-trajectory probabilities to a reward-induced Boltzmann distribution. 
\MODEL targets categorical graphs and jointly matches the graph-size marginal and conditional molecular distributions of a reward-tilted target under separate KL constraints.

\section{Conclusion}
\label{sec:conclusion}
We introduced \MODEL, a forward distribution matching framework for reward-based post-training of molecular graph diffusion models.
\MODEL matches a target distribution obtained by reward reweighting of valid terminal graphs, using the same sample weights in the denoising objective.
SDC jointly optimizes the structural marginal over graph sizes and the conditional molecular distributions at each size.
We derive the unique optimal teacher, prove condition-wise policy improvement, and establish an irreducible matching gap under a fixed graph-size prior.
Across molecular and polymer benchmarks, \MODEL achieves the strongest multi-property control with high chemical validity and generalizes to unseen targets from rare property combinations, demonstrating forward distribution matching as an effective alternative to reverse-policy optimization.

\bibliography{iclr2027_conference}
\bibliographystyle{iclr2027_conference}

\newpage
\appendix
\section{Appendix}
\label{sec:appendix}
\subsection{Additional Details on Dataset}
\label{app:dataset_details}

We evaluate \MODEL on a six-property small-molecule benchmark and a
four-property polymer benchmark. Table~\ref{tab:dataset_overview}
summarizes the data used for unconditional pretraining and conditional
post-training. All random subset construction and splitting use seed
42. The same frozen splits and train-derived normalization statistics
are used by all compared methods.

\begin{table*}[hbt]
\centering
\caption{
Overview of the datasets used by \MODEL.
PT uses molecular structures only and does not access target labels or conditional split identities.
}
\label{tab:dataset_overview}
\small
\setlength{\tabcolsep}{6pt}
\begin{tabular}{lcccc}
\toprule
\textbf{Benchmark}
& \textbf{Unlabeled PT}
& \textbf{Conditional Training}
& \textbf{Train/Val/Test}
& \textbf{Motif tokenizer} \\
\midrule
Small molecule
& 100,000
& 10,000
& 6,000/2,000/2,000
& V300-R80 \\
Polymer
& 12,792
& 553
& 337/107/109
& V1000-R80 \\
\bottomrule
\end{tabular}
\end{table*}

\subsubsection{Small-Molecule Benchmark}
\label{app:molecule_dataset}

The small-molecule benchmark is constructed from rQM9 v1~\citep{molguidance}, a corrected
version of QM9 with revised molecular graphs
\citep{qm9}.
Starting from 133,885 source entries, we remove 303 officially
identified problematic structures, 88 structures that fail RDKit
sanitization, 742 structures containing formal charges, and 370
entries belonging to duplicate canonical-SMILES groups. This produces
132,382 unique neutral molecular structures.

We remove explicit hydrogens and stereochemical annotations and
represent each molecule by a connected, canonical, non-isomeric
SMILES. From the cleaned collection, we sample 100,000 structures
without replacement using seed 42 for unconditional pretraining. The
10,000-molecule conditional benchmark is a fixed subset of this
structural pretraining corpus. Importantly, property labels and
conditional split identities are not used during pretraining. 

\paragraph{OOD targets.}
We additionally construct 200 compositional-OOD target conditions from
the 90,000 pretraining molecules outside the conditional benchmark.
For each of $\mu$, $\alpha$, HOMO, LUMO, and $C_v$, the 20th and 80th percentiles of the conditional training split define low, medium, and high regions. 
Their Cartesian combinations define joint property signatures.
We retain signatures represented by at most $0.1\%$ training molecules and select the eight rarest eligible signatures.
The selected signatures occur only two to five times in the 6,000-molecule conditional training split.
We randomly sample 25 targets from each signature using a deterministic seed schedule beginning at 42.
Every selected property value remains within its corresponding
training-set marginal range. SA is
computed after selection and included as the sixth target property.

\subsubsection{Polymer Benchmark}
\label{app:polymer_dataset}

We follow the polymer gas-permeability benchmark, preprocessing, and
evaluation protocol of GraphDiT~\citep{graphdit}. 
The benchmark contains
553 polymer repeat-unit records with jointly observed O$_2$, N$_2$,
and CO$_2$ permeability and an SA score. 
Repeat-unit SMILES
use wildcard atoms to mark polymerization sites.

Permeabilities are stored in Barrer, and model conditioning uses
train-split z-score normalization. Gas-property
errors are evaluated in $\log_{10}$ Barrer space. The polymer backbone
is initialized from CoMole's 12,792-polymer unconditional pretraining
corpus.

\subsubsection{Tokenizer}
\label{app:tokenizer}

We use domain-specific motif tokenizers based on the node-pair encoding
procedure of CoMole~\citep{comole}, where $V$ denotes the motif
vocabulary size and $R$ the ring-token budget. For polymers, we
directly reuse CoMole's V1000-R80 tokenizer learned from its pretraining corpus.

For the rQM9 graphs, we select a smaller V300-R80
tokenizer using the 100k unconditional pretraining
corpus. 
We evaluate the fixed grid
$V\in\{300,500,1000\}$ and $R\in\{30,50,80\}$ using representation-level
compression and reconstruction statistics. Increasing the ring budget
from 30 to 80 consistently shortens the motif sequences. At $R=80$,
increasing $V$ from 300 to 1000 reduces the mean motif count by only
$1.6\%$ and leaves its 95th percentile unchanged, while adding 700
prediction classes. 
We therefore choose V300-R80 as the smallest
tokenizer at this compression plateau. No conditional labels or
downstream validation performance are used in this selection.

\subsubsection{Target Statistics and Normalization}
\label{app:dataset_target_statistics}

For each property $y$, the model receives the standardized condition
\[
\widetilde y
=
\frac{y-\mu_y^{\mathrm{train}}}
     {\sigma_y^{\mathrm{train}}},
\]
where both statistics are computed exclusively from the corresponding
training split. Table~\ref{tab:target_statistics} reports the raw
range over all splits and the train-split normalization statistics.
\Cref{fig:dataset_target_distributions} shows that the fixed
splits cover similar regions of each target distribution.

\begin{table*}[t]
\centering
\caption{
Target ranges and train-split normalization statistics.
Gas-permeability statistics are computed in the raw Barrer scale.
}
\label{tab:target_statistics}
\small
\setlength{\tabcolsep}{5pt}
\begin{tabular}{lllrr}
\toprule
\textbf{Benchmark}
& \textbf{Target}
& \textbf{Unit}
& \textbf{Full range}
& \textbf{Train mean/std} \\
\midrule
Small molecule
& SA
& - & 1.111-7.750 & 4.236 / 0.924 \\
& $\mu$
& D & 0-9.071 & 2.665 / 1.433 \\
& $\alpha$
& $a_0^3$ & 13.21-127.46 & 75.274 / 8.226 \\
& HOMO
& Ha & $-0.3963$-$-0.1538$ & $-0.2404$ / 0.0219 \\
& LUMO
& Ha & $-0.1466$-0.1171 & 0.0128 / 0.0469 \\
& $C_v$
& $\mathrm{cal\,mol^{-1}\,K^{-1}}$
& 6.469-46.306 & 31.663 / 4.118 \\
\midrule
Polymer
& SA
& - & 2.48-7.62 & 4.036 / 0.990 \\
& O$_2$
& Barrer & $2.8{\times}10^{-4}$-$1.87{\times}10^4$
& 387.180 / 1,678.333 \\
& N$_2$
& Barrer & $1.6{\times}10^{-4}$-$1.66{\times}10^4$
& 253.367 / 1,281.755 \\
& CO$_2$
& Barrer & $1.8{\times}10^{-3}$-$4.70{\times}10^4$
& 1,281.308 / 4,771.168 \\
\bottomrule
\end{tabular}
\end{table*}

\begin{figure*}[t]
\centering
\includegraphics[width=\textwidth]{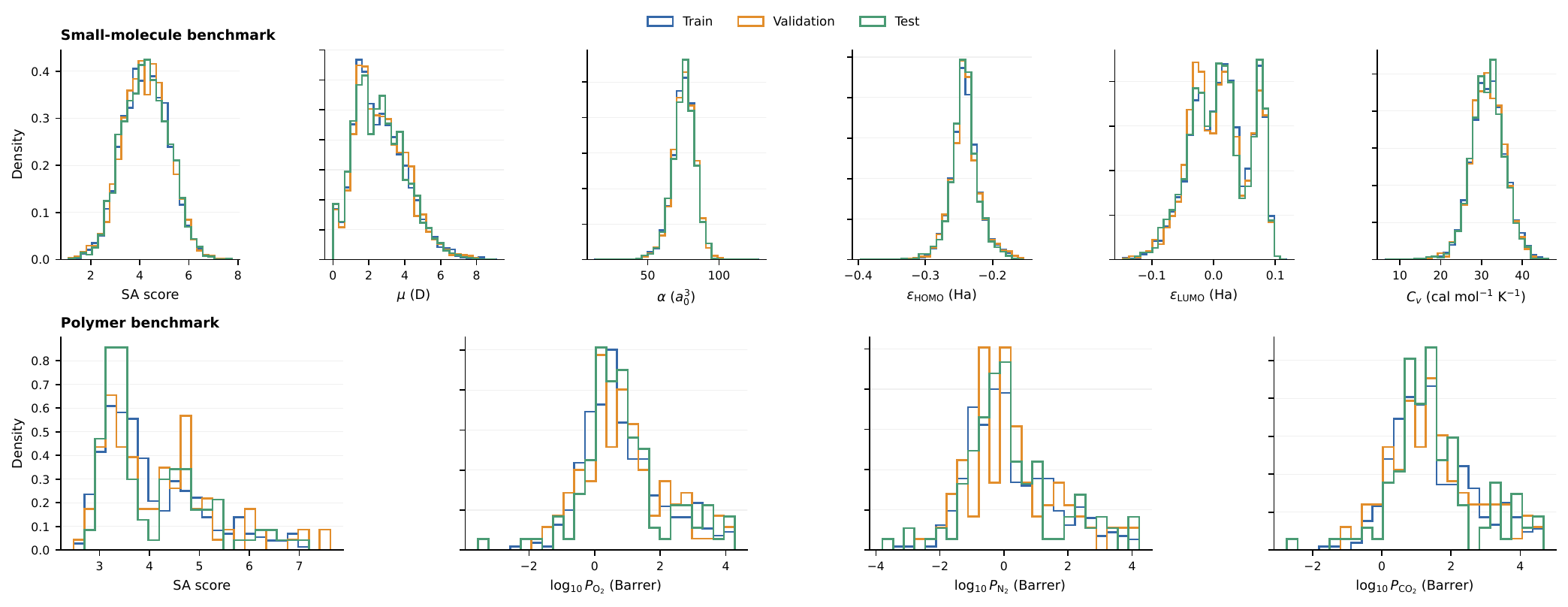}
\caption{
Target distributions for the small-molecule and polymer benchmarks,
shown separately for the training, validation, and test splits.
Polymer gas permeabilities are visualized in $\log_{10}$ Barrer,
consistent with the evaluation metric.
}
\label{fig:dataset_target_distributions}
\end{figure*}

\subsubsection{Structural Statistics}
\label{app:dataset_structural_statistics}

Table~\ref{tab:dataset_structural_statistics} summarizes the molecular
scale and motif-graph scale of the two conditional benchmarks. Real
atom counts exclude polymer wildcard atoms. Active motif units are the
structural variable $N$ used by the motif-aware diffusion backbone and
controlled by SDC.

As shown in
\cref{fig:dataset_structural_distributions}, motif tokenization induces
a substantially different active-unit distribution from the real-atom
count. This distinction is directly relevant to \MODEL because the
active-unit count is sampled before denoising and forms the structural
marginal controlled by SDC.

\begin{table*}[t]
\centering
\caption{
Structural statistics of the conditional benchmarks.
Values are reported as mean/maximum.
}
\label{tab:dataset_structural_statistics}
\small
\setlength{\tabcolsep}{6pt}
\begin{tabular}{lccccc}
\toprule
\textbf{Benchmark}
& \textbf{Real atoms}
& \textbf{Hetero atoms}
& \textbf{Rings}
& \textbf{Mol.\ weight}
& \textbf{Active motif units} \\
\midrule
Small molecule
& 8.80/9
& 2.44/7
& 1.74/8
& 122.72/152.04
& 4.84/9 \\
Polymer
& 27.97/48
& 5.28/17
& 3.68/11
& 391.78/866.20
& 6.66/46 \\
\bottomrule
\end{tabular}
\end{table*}

\begin{figure}[t]
\centering
\includegraphics[width=0.95\linewidth]{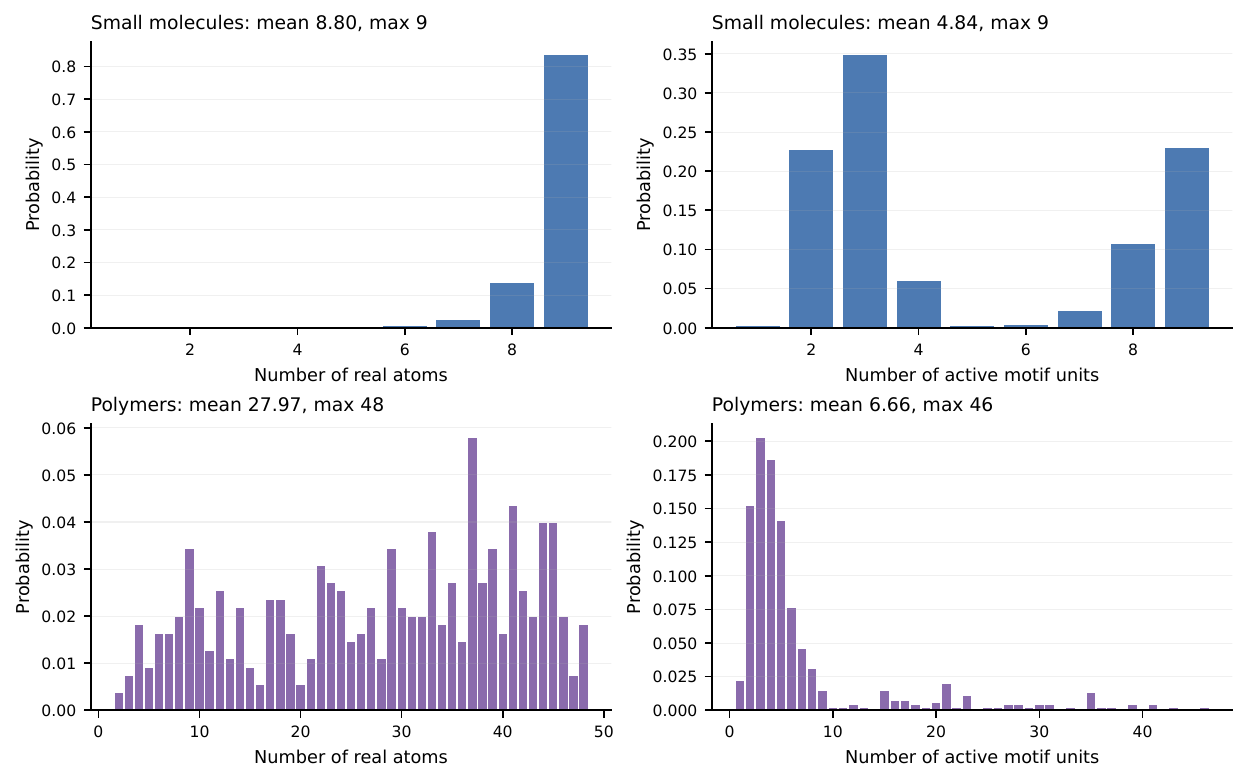}
\caption{
Distributions of real-atom counts and active motif-unit counts in the
two conditional benchmarks. Polymer wildcard atoms are excluded from
the real-atom count. The active motif-unit count corresponds to the
structural variable $N$ used by \MODEL.
}
\label{fig:dataset_structural_distributions}
\end{figure}

\subsection{Additional Theory}
\label{app:additional_theory}
Throughout this appendix, $\mathbb E_{\mathbf c}$ denotes the empirical
average over target conditions. In condition-wise derivations, we
suppress conditioning on $\mathbf c$ when no ambiguity arises.

\subsubsection{Validity Conditioning and Structural Support}
\label{app:valid_structural_support}

Define the count-conditional validity probability

$$
v_k(n,\mathbf c)
:=
\Pr_{S\sim\pi_{\theta_k}(\cdot\mid n,\mathbf c)}
\left(
(n,S)\in\mathcal G_{\mathrm{val}}
\right).
$$

From
\cref{eq:current_graph_distribution,eq:valid_reference},
\begin{equation}
\label{eq:valid_reference_marginal}
\bar p_k(n\mid\mathbf c)
=
\frac{
p_k(n\mid\mathbf c)\,v_k(n,\mathbf c)
}{
\sum_{n'}
p_k(n'\mid\mathbf c)\,v_k(n',\mathbf c)
}.
\end{equation}
Thus, validity conditioning reweights \(p_k\) according to the
count-dependent validity probability \(v_k\).
The two marginals coincide if and only if $v_k(n,\mathbf c)$ is
constant over the support of $p_k(\cdot\mid\mathbf c)$.
Moreover,
$$
\bar p_k(\cdot\mid\mathbf c)
\ll
p_k(\cdot\mid\mathbf c)
\ll
p_0.
$$

\subsubsection{Proof of Forward Distribution Matching}
\label{app:proof_reward_to_denoising}

\begin{proof}
Fix a target condition \(\mathbf c\) and define

$$
H_\theta(G,\mathbf c)
:=
\mathbb E_{\substack{
t\sim\operatorname{Unif}(\{1,\ldots,T\}),\\
S_t\sim q_t(\cdot\mid S_0,m)
}}
\left[
\ell_{\mathrm{den}}^\theta
(S_0,S_t,m,t,\mathbf c)
\right].
$$

Using

$$
w_k^\mu(G,\mathbf c)
=
\frac{
\mu_k(G\mid\mathbf c)
}{
\bar\pi_k(G\mid\mathbf c)
},
$$

the change-of-measure identity gives

$$
\begin{aligned}
&
\mathbb E_{G\sim\bar\pi_k(\cdot\mid\mathbf c)}
\left[
w_k^\mu(G,\mathbf c)
H_\theta(G,\mathbf c)
\right]
\\
&=
\mathbb E_{G\sim\mu_k(\cdot\mid\mathbf c)}
\left[
H_\theta(G,\mathbf c)
\right],
\end{aligned}
$$

which proves \cref{eq:forward_denoising_projection}.

Applying the same identity to the structural log loss gives

$$
\begin{aligned}
&
\mathbb E_{G\sim\bar\pi_k(\cdot\mid\mathbf c)}
\left[
w_k^\mu(G,\mathbf c)
\bigl(-\log p_\phi(N\mid\mathbf c)\bigr)
\right]
\\
&=
\mathbb E_{N\sim\mu_{N,k}(\cdot\mid\mathbf c)}
\left[
-\log p_\phi(N\mid\mathbf c)
\right],
\end{aligned}
$$

which proves \cref{eq:structural_projection_loss}. Averaging over
target conditions preserves both identities.
\end{proof}

\subsubsection{Proper-Scoring Interpretation}
\label{app:proper_scoring}

Consider the teacher-corruption sampling scheme

$$
\mathbf c
\sim
\text{the empirical target-condition distribution},
\qquad
G=(N,S_0)
\sim
\mu_k(\cdot\mid\mathbf c),
$$

$$
t\sim\operatorname{Unif}(\{1,\ldots,T\}),
\qquad
S_t\sim q_t(\cdot\mid S_0,m).
$$

Let \(U=(S_t,m,t,\mathbf c)\). For an active entry \(a\) of
\(Y\in\{X,E,P\}\), let
\(d_{\theta,Y,a}(\cdot\mid U)\) denote the corresponding categorical
prediction and define the teacher-induced posterior marginal

$$
r_{Y,a}^{\mu}(\cdot\mid U)
:=
\Pr(Y_{0,a}=\cdot\mid U)
$$

under the sampling scheme above. Its expected cross-entropy decomposes
as

$$
\begin{aligned}
&
\mathbb E
\left[
-\log d_{\theta,Y,a}(Y_{0,a}\mid U)
\right]
\\
&=
\mathbb E_U
\left[
H\left(r_{Y,a}^{\mu}(\cdot\mid U)\right)
+
D_{\mathrm{KL}}
\left(
r_{Y,a}^{\mu}(\cdot\mid U)
\Vert
d_{\theta,Y,a}(\cdot\mid U)
\right)
\right].
\end{aligned}
$$

Summing over the active entries of \(X\), \(E\), and \(P\) shows that,
under simultaneous realizability across target conditions, every
population minimizer of
\(\mathcal L_{S,k}^{\mu}(\theta)\) recovers the corresponding
teacher-induced posterior marginals almost surely.

Likewise, the structural objective satisfies

$$
\begin{aligned}
\mathcal L_{N,k}^{\mu}(\phi)
=
\mathbb E_{\mathbf c}
\Big[
&H\left(
\mu_{N,k}(\cdot\mid\mathbf c)
\right)
\\
&+
D_{\mathrm{KL}}
\left(
\mu_{N,k}(\cdot\mid\mathbf c)
\Vert
p_\phi(\cdot\mid\mathbf c)
\right)
\Big].
\end{aligned}
$$

Hence, under simultaneous realizability, every population minimizer of
\(\mathcal L_{N,k}^{\mu}(\phi)\) recovers
\(\mu_{N,k}(\cdot\mid\mathbf c)\) for every target condition with
positive empirical mass.

\subsubsection{Proof of Proposition~\ref{prop:structural_marginal_gap}}
\label{app:structural_gap}

\begin{proof}
Fix $\mathbf c$ and suppress it from the notation. For any conditional
distribution $\widetilde\pi(S\mid N)$ for which the divergence is
finite, the hierarchical factorizations give
\[
\log
\frac{
\mu_k(N,S)
}{
p_k(N)\widetilde\pi(S\mid N)
}
=
\log
\frac{
\mu_{N,k}(N)
}{
p_k(N)
}
+
\log
\frac{
\mu_{S,k}(S\mid N)
}{
\widetilde\pi(S\mid N)
}.
\]
Taking expectation under $\mu_k(N,S)$ yields
\cref{eq:structural_kl_decomposition}. The conditional KL term is
nonnegative and becomes zero when
\[
\widetilde\pi(\cdot\mid N)
=
\mu_{S,k}(\cdot\mid N)
\]
for $\mu_{N,k}$-almost every $N$. Therefore, taking the infimum
over conditional molecular distributions gives
\cref{eq:irreducible_structural_gap}.
\end{proof}

\subsubsection{Proof of \cref{thm:structured_teacher}}
\label{app:structural_teacher}

\begin{proof}
The Lagrangian of \cref{eq:sdc_trust_region_problem} is
\begin{equation}
\label{eq:sdc_lagrangian}
\begin{aligned}
\mathfrak L(\rho,\tau_N,\tau_S)
&=
\mathbb E_{\mathbf c}
\mathbb E_{G\sim\rho(\cdot\mid\mathbf c)}
\left[R_{\mathbf c}(G)\right]
-
\tau_N\!\left(\mathcal C_N(\rho)-\epsilon_N\right)
-
\tau_S\!\left(\mathcal C_S(\rho)-\epsilon_S\right).
\end{aligned}
\end{equation}
Fix \(\tau_N,\tau_S>0\). The terms
\(\tau_N\epsilon_N+\tau_S\epsilon_S\) do not depend on \(\rho\),
and the remaining maximization separates across target conditions.
All condition-wise statements below concern conditions in the support
of the law defining \(\mathbb E_{\mathbf c}\). For a fixed such
\(\mathbf c\), write
\begin{align}
\mathcal J_{\mathbf c}(\rho_N,\rho_S)
&:=
\mathbb E_{N\sim\rho_N(\cdot\mid\mathbf c)}
\Big[
\mathbb E_{S\sim\rho_S(\cdot\mid N,\mathbf c)}
[R_{\mathbf c}(N,S)]
\notag\\
&\hspace{35mm}
-
\tau_S
D_{\mathrm{KL}}\!\left(
\rho_S(\cdot\mid N,\mathbf c)
\Vert
\bar\pi_k(\cdot\mid N,\mathbf c)
\right)
\Big]
\notag\\
&\quad-
\tau_N
D_{\mathrm{KL}}\!\left(
\rho_N(\cdot\mid\mathbf c)
\Vert
\bar p_k(\cdot\mid\mathbf c)
\right).
\label{eq:conditionwise_sdc_lagrangian}
\end{align}

For
\(n\in\operatorname{supp}\bar p_k(\cdot\mid\mathbf c)\), define
\[
\widetilde\pi_{S,k}(S\mid n,\mathbf c)
:=
\frac{
\bar\pi_k(S\mid n,\mathbf c)
\exp\!\left(R_{\mathbf c}(n,S)/\tau_S\right)
}{
Z_{S,k}(n,\mathbf c;\tau_S)
}.
\]
The Gibbs variational identity gives, for every
\(\rho_S(\cdot\mid n,\mathbf c)\ll
\bar\pi_k(\cdot\mid n,\mathbf c)\),
\begin{align}
&
\mathbb E_{\rho_S(\cdot\mid n,\mathbf c)}
[R_{\mathbf c}(n,S)]
-
\tau_S
D_{\mathrm{KL}}\!\left(
\rho_S(\cdot\mid n,\mathbf c)
\Vert
\bar\pi_k(\cdot\mid n,\mathbf c)
\right)
\notag\\
&\qquad=
\tau_S\log Z_{S,k}(n,\mathbf c;\tau_S)
-
\tau_S
D_{\mathrm{KL}}\!\left(
\rho_S(\cdot\mid n,\mathbf c)
\Vert
\widetilde\pi_{S,k}(\cdot\mid n,\mathbf c)
\right).
\label{eq:conditional_sdc_gibbs_identity}
\end{align}
Thus the conditional maximizer is unique and equals
\(\widetilde\pi_{S,k}(\cdot\mid n,\mathbf c)\).

Now define
\[
\widetilde\pi_{N,k}(n\mid\mathbf c)
:=
\frac{
\bar p_k(n\mid\mathbf c)
Z_{S,k}(n,\mathbf c;\tau_S)^{\tau_S/\tau_N}
}{
Z_{N,k}(\mathbf c;\tau_N,\tau_S)
}.
\]
A second application of the Gibbs variational identity yields
\begin{align}
&
\mathbb E_{N\sim\rho_N(\cdot\mid\mathbf c)}
\!\left[
\tau_S\log Z_{S,k}(N,\mathbf c;\tau_S)
\right]
-
\tau_N
D_{\mathrm{KL}}\!\left(
\rho_N(\cdot\mid\mathbf c)
\Vert
\bar p_k(\cdot\mid\mathbf c)
\right)
\notag\\
&\qquad=
\tau_N\log Z_{N,k}(\mathbf c;\tau_N,\tau_S)
-
\tau_N
D_{\mathrm{KL}}\!\left(
\rho_N(\cdot\mid\mathbf c)
\Vert
\widetilde\pi_{N,k}(\cdot\mid\mathbf c)
\right).
\label{eq:marginal_sdc_gibbs_identity}
\end{align}
Combining
\cref{eq:conditional_sdc_gibbs_identity,eq:marginal_sdc_gibbs_identity}
gives
\begin{align}
\mathcal J_{\mathbf c}(\rho_N,\rho_S)
&=
\tau_N\log Z_{N,k}(\mathbf c;\tau_N,\tau_S)
-
\tau_N
D_{\mathrm{KL}}\!\left(
\rho_N(\cdot\mid\mathbf c)
\Vert
\widetilde\pi_{N,k}(\cdot\mid\mathbf c)
\right)
\notag\\
&\quad-
\tau_S
\mathbb E_{N\sim\rho_N(\cdot\mid\mathbf c)}
D_{\mathrm{KL}}\!\left(
\rho_S(\cdot\mid N,\mathbf c)
\Vert
\widetilde\pi_{S,k}(\cdot\mid N,\mathbf c)
\right).
\label{eq:conditionwise_sdc_gap_identity}
\end{align}
Both divergence terms are nonnegative. Moreover,
\(\widetilde\pi_{N,k}\) is positive on
\(\operatorname{supp}\bar p_k(\cdot\mid\mathbf c)\).
Consequently, the unique maximizing joint distribution is
\[
\widetilde\pi_k(N,S\mid\mathbf c)
=
\widetilde\pi_{N,k}(N\mid\mathbf c)
\widetilde\pi_{S,k}(S\mid N,\mathbf c).
\]

It remains to connect this Lagrangian maximizer to the constrained
problem. Viewed as an optimization over the joint probabilities on the
finite reference support, the reward objective is linear and both KL
constraints are convex; convexity of \(\mathcal C_S\) follows from the
perspective form of conditional relative entropy. Furthermore,
\(\rho=\bar\pi_k\) satisfies
\(\mathcal C_N(\bar\pi_k)=\mathcal C_S(\bar\pi_k)=0\), and is therefore
strictly feasible because \(\epsilon_N,\epsilon_S>0\). Slater's
condition and the KKT conditions therefore apply. At the assumed
positive optimal dual variables, the primal optimizer must equal the
unique Lagrangian maximizer above. Identifying
\(\widetilde\pi_k=\pi_k^\star\) proves
\cref{eq:structural_conditional_teacher,eq:structural_marginal_teacher}
and uniqueness.
\end{proof}

\paragraph{Projection ratio and shared-temperature case.}
On the reference support, multiplying the two factorwise ratios gives
\[
\frac{\pi_k^\star(G\mid\mathbf c)}
{\bar\pi_k(G\mid\mathbf c)}
=
\frac{
\exp\!\left(R_{\mathbf c}(G)/\tau_S\right)
Z_{S,k}(N,\mathbf c;\tau_S)^{\tau_S/\tau_N-1}
}{
Z_{N,k}(\mathbf c;\tau_N,\tau_S)
},
\]
which is \cref{eq:structured_reward_weight}. If
\(\tau_N=\tau_S=\tau\), then
\[
Z_{N,k}(\mathbf c;\tau,\tau)
=
\sum_n \bar p_k(n\mid\mathbf c)
Z_{S,k}(n,\mathbf c;\tau)
=
\mathbb E_{G\sim\bar\pi_k(\cdot\mid\mathbf c)}
\!\left[\exp\!\left(R_{\mathbf c}(G)/\tau\right)\right].
\]
The factor
\(Z_{S,k}(N,\mathbf c;\tau)\) cancels between the structural and
conditional teacher factors, yielding
\cref{eq:shared_temperature_target}.

\subsubsection{Proof of
Corollary~\ref{cor:structured_teacher_improvement}}
\label{app:structured_teacher_improvement}

\begin{proof}
Fix \(\mathbf c\) and abbreviate
\[
P_N:=\pi_k^\star(N\mid\mathbf c),
\qquad
Q_N:=\bar p_k(N\mid\mathbf c),
\]
and, for each \(n\) in their common support,
\[
P_S^n:=\pi_k^\star(S\mid n,\mathbf c),
\qquad
Q_S^n:=\bar\pi_k(S\mid n,\mathbf c).
\]
Because the reward is real-valued on the finite reference support, the
teacher and reference factors have the same supports. Combining
\cref{eq:structural_conditional_teacher,eq:structural_marginal_teacher}
therefore gives, on that support,
\begin{equation}
\label{eq:structured_reward_log_ratio_identity}
R_{\mathbf c}(n,s)
=
\tau_N\log\frac{P_N(n)}{Q_N(n)}
+
\tau_S\log\frac{P_S^n(s)}{Q_S^n(s)}
+
\tau_N\log Z_{N,k}(\mathbf c;\tau_N,\tau_S).
\end{equation}
Taking the expectation of
\cref{eq:structured_reward_log_ratio_identity} first under
\(P_NP_S^N\), then under \(Q_NQ_S^N\), and subtracting cancels the
last term and yields the exact identity
\begin{equation}
\label{eq:conditionwise_exact_improvement}
\begin{aligned}
\Delta_k(\mathbf c)
&=
\tau_N
\left[
D_{\mathrm{KL}}(P_N\Vert Q_N)
+
D_{\mathrm{KL}}(Q_N\Vert P_N)
\right]
\\
&\quad+
\tau_S
\mathbb E_{N\sim P_N}
D_{\mathrm{KL}}(P_S^N\Vert Q_S^N)
\\
&\quad+
\tau_S
\mathbb E_{N\sim Q_N}
D_{\mathrm{KL}}(Q_S^N\Vert P_S^N).
\end{aligned}
\end{equation}
Dropping the two reverse-KL terms proves
\cref{eq:conditionwise_reward_improvement}. If
\(\pi_k^\star(\cdot\mid\mathbf c)\neq
\bar\pi_k(\cdot\mid\mathbf c)\), then the KL chain rule implies
\[
D_{\mathrm{KL}}(P_N\Vert Q_N)
+
\mathbb E_{N\sim P_N}
D_{\mathrm{KL}}(P_S^N\Vert Q_S^N)
>0.
\]
Since \(\tau_N,\tau_S>0\), the lower bound in
\cref{eq:conditionwise_reward_improvement} is then strictly positive.

Averaging \cref{eq:conditionwise_exact_improvement} over target
conditions and retaining the two forward-KL terms gives
\[
\mathbb E_{\mathbf c}[\Delta_k(\mathbf c)]
\geq
\tau_N\mathcal C_N(\pi_k^\star)
+
\tau_S\mathcal C_S(\pi_k^\star).
\]
Both constraints are active by assumption, so
\(\mathcal C_N(\pi_k^\star)=\epsilon_N\) and
\(\mathcal C_S(\pi_k^\star)=\epsilon_S\), proving
\cref{eq:structured_reward_improvement}.

Finally,
\(\pi_k^\star\ll\bar\pi_k\) and
\(\bar\pi_k(\mathcal G_{\mathrm{val}}\mid\mathbf c)=1\) imply
\(\pi_k^\star(\mathcal G_{\mathrm{val}}\mid\mathbf c)=1\).
\end{proof}

\subsubsection{Finite-Sample Structured Projection}
\label{app:structured_projection}

We describe the estimator for one target condition \(\mathbf c\);
the training objectives and KL constraints are averaged over the
sampled target conditions. To include active-node-count exploration,
let
\[
b_k(G\mid\mathbf c)
=
b_{N,k}(N\mid\mathbf c)\,
\pi_{\theta_k}(S\mid N,\mathbf c)
\]
be a behavior distribution satisfying
\(p_k(\cdot\mid\mathbf c)\ll
b_{N,k}(\cdot\mid\mathbf c)\). The online update in the main text is
the special case \(b_{N,k}=p_k\). Draw \(B\) independent candidates
\[
N_i\sim b_{N,k}(\cdot\mid\mathbf c),
\qquad
S_{0,i}\sim\pi_{\theta_k}(\cdot\mid N_i,\mathbf c),
\qquad
G_i=(N_i,S_{0,i}).
\]
Define
\[
a_i(\mathbf c)
:=
\mathbf 1\{G_i\in\mathcal G_{\mathrm{val}}\}
\frac{p_k(N_i\mid\mathbf c)}
{b_{N,k}(N_i\mid\mathbf c)}.
\]
Whenever \(A(\mathbf c):=\sum_{j=1}^B a_j(\mathbf c)>0\), set
\[
\widehat r_{k,i}(\mathbf c)
:=
\frac{a_i(\mathbf c)}{A(\mathbf c)}.
\]
These self-normalized proposal weights form the empirical
valid-conditioned reference distribution. In particular, if
\(b_{N,k}=p_k\), they are uniform over the valid candidates.

Let
\[
\mathcal I_n(\mathbf c)
:=
\{i:a_i(\mathbf c)>0,\ N_i=n\},
\qquad
\widehat{\bar p}_k(n\mid\mathbf c)
:=
\sum_{i\in\mathcal I_n(\mathbf c)}
\widehat r_{k,i}(\mathbf c).
\]
For every empirically represented \(n\), define the empirical
reference conditional mass and conditional partition function by
\begin{align}
\widehat{\bar\pi}_k(i\mid n,\mathbf c)
&:=
\frac{\widehat r_{k,i}(\mathbf c)}
{\widehat{\bar p}_k(n\mid\mathbf c)},
\qquad i\in\mathcal I_n(\mathbf c),
\label{eq:empirical_reference_conditional}
\\
\widehat Z_{S,k}(n,\mathbf c;\tau_S)
&:=
\sum_{i\in\mathcal I_n(\mathbf c)}
\widehat{\bar\pi}_k(i\mid n,\mathbf c)
\exp\!\left(\frac{R_{\mathbf c}(G_i)}{\tau_S}\right).
\label{eq:empirical_conditional_partition}
\end{align}
Here sampled indices are treated as empirical atoms; repeated graphs
may equivalently be aggregated. The empirical structural partition
function is
\begin{equation}
\label{eq:empirical_structural_partition}
\widehat Z_{N,k}(\mathbf c;\tau_N,\tau_S)
:=
\sum_{n:\widehat{\bar p}_k(n\mid\mathbf c)>0}
\widehat{\bar p}_k(n\mid\mathbf c)
\widehat Z_{S,k}(n,\mathbf c;\tau_S)^{\tau_S/\tau_N}.
\end{equation}

The plug-in estimate of the density ratio in
\cref{eq:structured_reward_weight} is
\begin{equation}
\label{eq:empirical_structured_density_ratio}
\widehat w_k(G_i,\mathbf c)
:=
\frac{
\exp\!\left(R_{\mathbf c}(G_i)/\tau_S\right)
\widehat Z_{S,k}(N_i,\mathbf c;\tau_S)^{
\tau_S/\tau_N-1
}
}{
\widehat Z_{N,k}(\mathbf c;\tau_N,\tau_S)
},
\qquad a_i(\mathbf c)>0.
\end{equation}
The corresponding empirical teacher mass assigned to candidate \(i\)
is
\begin{equation}
\label{eq:empirical_structured_weight}
\widehat\omega_{k,i}(\mathbf c)
:=
\widehat r_{k,i}(\mathbf c)\,
\widehat w_k(G_i,\mathbf c).
\end{equation}
Indeed,
\begin{align*}
&\sum_{i:a_i(\mathbf c)>0}
\widehat r_{k,i}(\mathbf c)
\exp\!\left(R_{\mathbf c}(G_i)/\tau_S\right)
\widehat Z_{S,k}(N_i,\mathbf c;\tau_S)^{
\tau_S/\tau_N-1
}
\\
&\qquad=
\sum_n
\widehat{\bar p}_k(n\mid\mathbf c)
\widehat Z_{S,k}(n,\mathbf c;\tau_S)^{\tau_S/\tau_N}
=
\widehat Z_{N,k}(\mathbf c;\tau_N,\tau_S),
\end{align*}
so \(\sum_i\widehat\omega_{k,i}(\mathbf c)=1\).
The induced empirical teacher factors are
\begin{align}
\widehat\pi_{N,k}^\star(n\mid\mathbf c)
&=
\frac{
\widehat{\bar p}_k(n\mid\mathbf c)
\widehat Z_{S,k}(n,\mathbf c;\tau_S)^{\tau_S/\tau_N}
}{
\widehat Z_{N,k}(\mathbf c;\tau_N,\tau_S)
},
\label{eq:empirical_structural_teacher}
\\
\widehat\pi_k^\star(i\mid n,\mathbf c)
&=
\frac{
\widehat{\bar\pi}_k(i\mid n,\mathbf c)
\exp\!\left(R_{\mathbf c}(G_i)/\tau_S\right)
}{
\widehat Z_{S,k}(n,\mathbf c;\tau_S)
},
\qquad i\in\mathcal I_n(\mathbf c).
\label{eq:empirical_conditional_teacher}
\end{align}
Equivalently,
\(\widehat\omega_{k,i}(\mathbf c)
=
\widehat\pi_{N,k}^\star(N_i\mid\mathbf c)
\widehat\pi_k^\star(i\mid N_i,\mathbf c)\).

\paragraph{Empirical projection objectives.}
For each valid candidate, sample
\[
t_i\sim\operatorname{Unif}\{1,\ldots,T\},
\qquad
S_{t_i}\sim q_{t_i}(\cdot\mid S_{0,i},m_i).
\]
Using the same empirical teacher mass for both projections gives
\begin{align}
\widehat{\mathcal L}_{S,k}(\theta;\mathbf c)
&=
\sum_{i:a_i(\mathbf c)>0}
\widehat\omega_{k,i}(\mathbf c)\,
\ell_{\mathrm{den}}^\theta
(S_{0,i},S_{t_i},m_i,t_i,\mathbf c),
\label{eq:empirical_sdc_denoising_loss}
\\
\widehat{\mathcal L}_{N,k}(\phi;\mathbf c)
&=
-
\sum_{i:a_i(\mathbf c)>0}
\widehat\omega_{k,i}(\mathbf c)
\log p_\phi(N_i\mid\mathbf c).
\label{eq:empirical_sdc_structural_loss}
\end{align}
The minibatch objectives are obtained by averaging these expressions
over target conditions.

\paragraph{Dual estimation.}
Suppose a minibatch contains target conditions
\(\mathbf c_1,\ldots,\mathbf c_M\), each with a nonempty valid
candidate bank. For a fixed positive pair \((\tau_N,\tau_S)\), define
\begin{align}
\widehat{\mathcal C}_N(\tau_N,\tau_S)
&:=
\frac{1}{M}\sum_{m=1}^M
D_{\mathrm{KL}}\!\left(
\widehat\pi_{N,k}^\star(\cdot\mid\mathbf c_m)
\Vert
\widehat{\bar p}_k(\cdot\mid\mathbf c_m)
\right),
\label{eq:empirical_structural_kl}
\\
\widehat{\mathcal C}_S(\tau_N,\tau_S)
&:=
\frac{1}{M}\sum_{m=1}^M
\mathbb E_{N\sim
\widehat\pi_{N,k}^\star(\cdot\mid\mathbf c_m)}
D_{\mathrm{KL}}\!\left(
\widehat\pi_k^\star(\cdot\mid N,\mathbf c_m)
\Vert
\widehat{\bar\pi}_k(\cdot\mid N,\mathbf c_m)
\right).
\label{eq:empirical_conditional_kl}
\end{align}
Applying the two Gibbs variational identities to the empirical
reference distribution gives the empirical dual objective
\begin{equation}
\label{eq:empirical_sdc_dual}
\widehat d_k(\tau_N,\tau_S)
:=
\tau_N\epsilon_N+\tau_S\epsilon_S
+
\frac{1}{M}\sum_{m=1}^M
\tau_N
\log\widehat Z_{N,k}
(\mathbf c_m;\tau_N,\tau_S).
\end{equation}
Thus the shared temperatures are estimated by minimizing
\(\widehat d_k\) over \(\tau_N,\tau_S>0\). By the envelope theorem,
\begin{equation}
\label{eq:empirical_sdc_dual_gradient}
\frac{\partial\widehat d_k}{\partial\tau_N}
=
\epsilon_N-\widehat{\mathcal C}_N,
\qquad
\frac{\partial\widehat d_k}{\partial\tau_S}
=
\epsilon_S-\widehat{\mathcal C}_S.
\end{equation}
Consequently, an interior dual optimum enforces the two empirical KL
budgets. The same pair of temperatures is shared across conditions
because the constraints in \cref{eq:sdc_trust_region_problem} are
condition-averaged.

If \(A(\mathbf c)=0\), that condition contributes no update in the
current iteration. Counts absent from the valid candidate bank receive
zero empirical teacher mass. Under the finite-support assumptions of
\cref{thm:structured_teacher}, the stated proposal-support condition,
and integrability of the projection losses, the self-normalized
partition and objective estimates converge almost surely to their
population counterparts for fixed positive temperatures as the
candidate-bank size tends to infinity.

\subsection{Additional Details on Experiments}
\label{app:experimental_details}

\subsubsection{Experimental Setup}
\label{app:experimental_setup}

Each training and evaluation job is run on a single NVIDIA H100 GPU.
We use the graph diffusion transformer backbone with 12 transformer layers,
hidden dimension 1024, 16 attention heads, and an MLP expansion ratio
of 4. The diffusion process uses 500 steps, a cosine noise schedule,
and marginal transition kernels.
PT and SFT use learning rate $2\times10^{-5}$. Each SFT
model is initialized from its corresponding PT checkpoint, and
\MODEL is initialized from the resulting SFT checkpoint.

The forward distribution matching post-training runs for 10 online rounds. Each round collects $K=32$
candidates per training condition and uses the symmetric trust-region
budgets $\epsilon_N=\epsilon_S=0.035$. The denoiser and structural
controller are each updated for 20 epochs per round using learning
rates $2\times10^{-6}$ and $10^{-4}$, respectively. The controller is
a two-hidden-layer MLP of width 64. Dataset-specific epoch and batch
configurations are reported in \cref{tab:training_configuration}.
Finite-sample teacher construction and proposal correction are
described in \cref{app:structured_projection}.

\begin{table}[hbt]
\centering
\caption{
Dataset-specific training configurations.
}
\label{tab:training_configuration}
\small
\begin{tabular}{lcc}
\toprule
\textbf{Configuration} & \textbf{Molecule} & \textbf{Polymer} \\
\midrule
PT epochs                 & 200   & 400  \\
PT batch size             & 256   & 64   \\
SFT epochs                & 200   & 1200 \\
SFT batch size            & 256   & 128  \\
Online rounds             & 10    & 10   \\
Candidates per condition  & 32    & 32   \\
Graph-update batch size   & 256   & 64   \\
Controller batch size     & 1024  & 64   \\
Test samples              & 10,000 & 10,000 \\
\bottomrule
\end{tabular}
\end{table}

\paragraph{Reward definition.}
Let $\widehat y_j(G)$ and $c_j$ denote the evaluated and target values
of property $j$, respectively. Define the property-specific evaluation
transform
\begin{equation}
h_j(y)
:=
\begin{cases}
\log_{10} y,
&
\text{$j$ is a polymer gas-permeability property},
\\
y,
&
\text{otherwise}.
\end{cases}
\end{equation}
For both benchmarks, GraphFDM uses the normalized absolute error
\begin{equation}
e_j(G,\mathbf c)
:=
\frac{
\left|
h_j\!\left(\widehat y_j(G)\right)-h_j(c_j)
\right|
}{
\sigma_j^{\mathrm{tr}}
},
\end{equation}
where $\sigma_j^{\mathrm{tr}}$ is the standard deviation of
$h_j(y_j)$ on the corresponding conditional training split. The
terminal reward is
\begin{equation}
R_{\mathbf c}(G)
=
-\frac{1}{d}
\sum_{j=1}^{d}
e_j(G,\mathbf c),
\end{equation}
where $d=6$ for small molecules and $d=4$ for polymers. Thus, all
target properties contribute equally after property-wise
normalization.

\paragraph{Baselines.}
We use author-released implementations and recommended configurations
whenever available. Otherwise, we reproduce the reported architecture
and training protocol and tune unspecified hyperparameters on the
validation set using Optuna~\citep{akiba2019optuna}. Implementations are
adapted to support our data and evaluation interfaces. All methods
use the same splits, target conditions, property normalization,
validity criteria, and property evaluators.

\subsubsection{Robustness Across Seeds}
\label{app:multi_seed}
To evaluate robustness to training stochasticity, we repeat the five leading methods from the main comparison (\cref{tab:main_results_mol}): \MODEL, GraphGRPO,
GDPO, CoMole, and DiGress, in order of their single-seed Avg.\ Rank.
We use three training seeds, $0,1,2$, where seed $0$ corresponds to the main-table run. The data split, property oracle,
and test targets are fixed across all runs. For post-training methods, we keep the same reference checkpoint and rerun only the
online post-training stage. 

\begin{table*}[t]
\centering
\small
\setlength{\tabcolsep}{6pt}
\renewcommand{\arraystretch}{1.15}

\caption{
\textbf{Multi-seed results of the top-5 models on QM9 multi-conditional generation.}
Results on six properties (synthetic score and five QM9 properties) are reported as mean$\pm$std.
Overall Rank aggregates property-wise MAE ranks across different units and scales.
Best results are in \BEST{red}.
}
\label{tab:multi_seed}

\textbf{(a) Validity and Distribution Learning}

\vspace{3pt}

\begin{tabular}{lcccc}
\toprule
\textbf{Model}
& \multicolumn{1}{c}{\textbf{Validity}}
& \multicolumn{3}{c}{\textbf{Distribution Learning}} \\
\cmidrule(lr){2-2}
\cmidrule(lr){3-5}
& Valid. (raw) $\uparrow$
& Div. $\uparrow$
& Sim. $\uparrow$
& Dis. $\downarrow$ \\
\midrule

DiGress
& 0.914$\pm$0.004 (0.827$\pm$0.020)
& \BEST{0.919$\pm$0.002}
& 0.943$\pm$0.002
& 4.742$\pm$0.142\\

CoMole
& 0.937$\pm$0.037 (0.831$\pm$0.067)
& 0.916$\pm$0.001
& 0.930$\pm$0.010
& 4.253$\pm$0.019 \\

GDPO
& 0.914$\pm$0.003 (0.849$\pm$0.010)
& 0.917$\pm$0.001
& \BEST{0.950$\pm$0.003}
& 0.795$\pm$0.087 \\

GraphGRPO
& 0.973$\pm$0.003 (0.836$\pm$0.015)
& 0.910$\pm$0.001
& 0.914$\pm$0.016
& 2.953$\pm$0.158 \\

\midrule

\MODEL{}
& \BEST{0.994$\pm$0.000 (0.959$\pm$0.001)}
& 0.916$\pm$0.001
& 0.907$\pm$0.007
& \BEST{0.612$\pm$0.004} \\

\bottomrule
\end{tabular}

\vspace{8pt}

\textbf{(b) Condition Control}

\vspace{3pt}

\setlength{\tabcolsep}{5pt}
\resizebox{\linewidth}{!}{%
\begin{tabular}{lccccccc}
\toprule

\textbf{Model}
& Synth. $\downarrow$
& $\mu$ $\downarrow$
& $\alpha$ $\downarrow$
& $\varepsilon_{\mathrm{HOMO}}$ $\downarrow$
& $\varepsilon_{\mathrm{LUMO}}$ $\downarrow$
& $C_v$ $\downarrow$
& Overall Rank $\downarrow$ \\

\midrule

DiGress
& 1.618$\pm$0.016
& 0.915$\pm$0.023
& 4.108$\pm$0.199
& 0.012$\pm$0.000
& 0.018$\pm$0.000
& 1.998$\pm$0.086
& 5 \\

CoMole
& 0.497$\pm$0.035
& 0.874$\pm$0.008
& 4.375$\pm$0.175
& 0.012$\pm$0.000
& 0.021$\pm$0.001
& 1.967$\pm$0.093
& 4 \\

GDPO
& 0.547$\pm$0.006
& 0.909$\pm$0.005
& 3.804$\pm$0.065
& 0.012$\pm$0.000
& 0.018$\pm$0.001
& 1.955$\pm$0.063
& 3 \\

GraphGRPO
& 0.577$\pm$0.007
& 0.855$\pm$0.014
& 3.569$\pm$0.118
& 0.011$\pm$0.001
& 0.018$\pm$0.000
& 2.010$\pm$0.040
& 2 \\
\midrule

\MODEL{}
& \BEST{0.361$\pm$0.002}
& \BEST{0.704$\pm$0.009}
& \BEST{3.422$\pm$0.162}
& \BEST{0.009$\pm$0.000}
& \BEST{0.015$\pm$0.000}
& \BEST{1.486$\pm$0.019}
& \BEST{1} \\

\bottomrule
\end{tabular}
}
\end{table*}

As shown in \cref{tab:multi_seed}, the Overall Rank remains stable across seeds, with \MODEL achieving the best overall
performance and low variance.

\subsubsection{Robustness Across oracles}
\label{app:oracle}

We fit Random Forest (RF), Support Vector Regression (SVR), and Gaussian Process Regression (GP) property oracles using all labeled molecules available in each task, with the goal of approximating the underlying property functions as accurately as possible.
The resulting predictors are treated as fixed surrogate property evaluators.
\Cref{tab:oracle_quality} reports their fitting accuracy.
RF achieves the lowest fitting MAE and is therefore used as the reward oracle and primary evaluator throughout our experiments.
\begin{table*}[t]
\centering
\small
\setlength{\tabcolsep}{4pt}
\caption{
\textbf{Training Performance of Oracle Methods:}
we train the models on all polymers or small
molecules in a task to simulate the Oracle. Results from the random forest model are in \textbf{bold} because it has the lowest training MAE.
}
\label{tab:oracle_quality}
\resizebox{0.8\textwidth}{!}{%
\begin{tabular}{llcccccccc}
\toprule
\textbf{Oracle} & \textbf{Metric}
& \textbf{$O_2$} & \textbf{$N_2$} & \textbf{$CO_2$}
& \textbf{$\mu$} & \textbf{$\alpha$}
& \textbf{$\varepsilon_{\mathrm{HOMO}}$}
& \textbf{$\varepsilon_{\mathrm{LUMO}}$}
& \textbf{$C_v$} \\
\midrule
RF  & MAE $\downarrow$
& \textbf{0.412} & \textbf{0.456} & \textbf{0.421}
& \textbf{0.194} & \textbf{1.343} & \textbf{0.002} & \textbf{0.002} & \textbf{0.540} \\
SVR & MAE $\downarrow$
& 0.480 & 0.539 & 0.488
& 0.862 & 4.476 & 0.010 & 0.019 & 2.056 \\
GP  & MAE $\downarrow$
& 0.872 & 0.989 & 0.874
& 1.153 & 6.304 & 0.015 & 0.039 & 3.148 \\
\bottomrule
\end{tabular}%
}
\end{table*}

To assess robustness to evaluator choice, we fix the generated molecules and
rescore the same samples using SVR and GP oracles trained on identical data.
As shown in \cref{tab:oracle_robustness}, \MODEL ranks first under all three evaluators, while a consistent partial ordering
\MODEL $\succ$ CoMole $\succ$ GraphDiT $\succ$ LSTM $\succ$ MELD
is preserved despite changes in absolute ranks.
These results indicate that the relative controllability gains are robust to
the choice of learned property oracle.
\begin{table*}[t]
\centering
\small
\setlength{\tabcolsep}{5pt}
\caption{
\textbf{Oracles for Generation Evaluation:} 
Models are ranked from 1 to 9 across the polymer benchmark by Avg.MAE under each evaluator.
Underlined entries preserve the same relative ordering across all evaluators.
}
\label{tab:oracle_robustness}
\resizebox{0.8\textwidth}{!}{%
\begin{tabular}{cccc}
\toprule
\textbf{Rank} & \textbf{Random Forest} & \textbf{Support Vector Regression} & \textbf{Gaussian Process}  \\
\midrule
\textbf{1} & \textbf{\underline{\MODEL{}}} & \textbf{\underline{\MODEL{}}} & \textbf{\underline{\MODEL{}}}\\
2 & \underline{CoMole} & \underline{CoMole} & \underline{CoMole} \\
3 & \underline{GraphDiT}  & GDSS & GDSS  \\
4 & CSGD & GDPO & GDPO \\
5 & GDSS  & \underline{GraphDiT}  & \underline{GraphDiT}  \\
6 & DeFoG  & \underline{LSTM}  & \underline{LSTM}  \\
7 & GDPO  & CSGD  & DeFoG\\
8 & \underline{LSTM}& DeFoG  & CSGD \\
9 & \underline{MELD}  & \underline{MELD} & \underline{MELD}  \\
\bottomrule
\end{tabular}%
}
\end{table*}

\subsubsection{Computational Cost}

We report online post-training wall-clock costs on a single NVIDIA H100, including candidate generation, oracle evaluation, optimization, and validation, but excluding base-model training and final evaluation.
TreeDiff is reported separately as an inference-time search method.

As shown in \cref{tab:normalized_cost}, CoMole and \MODEL share the same SFT initialization and backbone, providing a matched comparison of post-training cost.
\MODEL requires 16.65 vs.\ 30.37 hours on polymers
and 49.18 vs.\ 50.30 hours on molecules. 
Its endpoint-based updates reuse each clean sample through forward re-noising without retaining the sampled reverse trajectory. 
VIDD~\citep{vidd} and GDPO~\citep{gdpo} have lower total costs under their native configurations, reflecting differences in backbone cost and online budget.
For example, VIDD operates on GDSS~\citep{gdss}
score/SDE trajectories, so replacing its backbone would require reformulating its optimization objective.
\begin{table*}[t]
\centering
\small
\setlength{\tabcolsep}{4.5pt}
\caption{
\textbf{Online post-training and inference-time search cost.}
We report wall-clock cost on one H100, including sampling, oracle evaluation, optimization, and validation. 
Evaluations denote the
dominant model calls for each method's generative dynamics.
CoMole and \MODEL use the same SFT initialization and backbone.
}
\label{tab:normalized_cost}
\resizebox{\textwidth}{!}{%
\begin{tabular}{llrrrr}
\toprule
\textbf{Method}
& \textbf{Online unit}
& \textbf{Evaluations/unit}
& \textbf{Seconds/unit}
& \textbf{ms/evaluation}
& \textbf{Total hours} \\
\midrule
\multicolumn{6}{l}{\textit{Polymer}} \\
GDPO
& reverse trajectory &  500 & 1.312 & 2.625 & 12.28 \\
GraphGRPO
& group trajectory & 500 & 1.940 & 3.884 & 19.30 \\
VIDD
& score trajectory &  2,400 & 9.291 & 3.871 & 3.55 \\
CoMole
& reverse trajectory  & 500 & 0.291
& 0.582 & 30.37 \\
\textbf{\MODEL{}}
& clean endpoint  & 500
& 0.556 & 1.112 & 16.65 \\
TreeDiff$^\dagger$
& search output  & 26,327 & 49.314 & 1.873 & 136.98 \\
\midrule
\multicolumn{6}{l}{\textit{Molecule}} \\
GDPO
& reverse trajectory  & 500 & 0.083 & 0.165 & 13.79 \\
GraphGRPO
& group trajectory  & 500 & 2.191 & 4.380 & 59.21 \\
VIDD
& score trajectory  & 2,400 & 2.112 & 0.880 & 10.59 \\
CoMole
& reverse trajectory  & 500 & 0.302 & 0.604 & 50.30 \\
\textbf{\MODEL{}}
& clean endpoint & 500
& 0.092 & 0.184 & 49.18 \\
TreeDiff$^\dagger$
& search output  & 26,294 & 53.605 & 2.039 & 148.90 \\
\bottomrule
\end{tabular}%
}
\vspace{2pt}

\footnotesize
A reverse trajectory is a complete generative path ending in one
molecule. A group trajectory is evaluated relative to other trajectories
sampled for the same condition. A score trajectory is used for
score-based distillation. A clean endpoint retains only the terminal
molecule.
$^\dagger$TreeDiff performs inference-time search, requiring
approximately 26K model forwards per output.
\end{table*}

\paragraph{Oracle overhead.}
Property rewards for all models are evaluated by CPU-based random-forest oracles, with SAscore computed directly from molecular structure. 
Oracle evaluation takes 1.12\,ms per
polymer and 0.79\,ms per QM9 molecule, corresponding to approximately
12 and 152 seconds per GraphFDM round (overall 20 rounds), respectively. This constitutes
only a small fraction of the total online cost.

\subsection{Limitations and Future Work}
\MODEL is designed as a post-training framework for pretrained graph diffusion models and therefore operates within the structural support provided by the underlying backbone.
In particular, SDC reallocates probability across graph sizes represented by the reference model.
Extending the framework to support-expanding structural updates, such as adaptive graph growth or vocabulary expansion, is an interesting direction for future work.
Our experiments cover both small molecules and polymers, and extending \MODEL to broader chemical spaces and other
structured generative domains would further examine its generality.

The current benchmarks use learned property oracles to enable scalable reward evaluation during online post-training. 
High-fidelity simulations and experimental measurements incur substantial computational or experimental costs, limiting their use for evaluation of large candidate populations. 
Future work could investigate sample-efficient integration of physics-based simulators and experimental feedback.
Finally, we use a normalized multi-property reward to study the core distribution-matching mechanism.
Incorporating task-specific preferences,
constraints, or Pareto-aware objectives would further broaden the range of inverse-design settings addressed by \MODEL.

\end{document}